\pdfoutput=1
\documentclass{article}

\PassOptionsToPackage{numbers, compress}{natbib}
 \usepackage[final]{nips_2017}

\usepackage[utf8]{inputenc} 
\usepackage[T1]{fontenc}    
\usepackage{hyperref}       
\usepackage{url}            
\usepackage{booktabs}       
\usepackage{amsfonts}       
\usepackage{nicefrac}       
\usepackage{microtype}      
\usepackage{color}
\usepackage{amsmath, amsthm, amssymb}
\usepackage{bm}
\usepackage{graphicx}
\usepackage{multirow}
\usepackage{tabularx}
\usepackage{booktabs}
\usepackage[ruled,vlined]{algorithm2e}
\usepackage{algorithmic}
\usepackage{float}
\usepackage{wrapfig}

\newcommand{\pol}[0]{\pmb{\pi}}
\newcommand{\cpol}[0]{\pmb{\mu}}
\newcommand{\dd}{\frac{\partial}{\partial \theta_i}}

\newcommand{\insertfigure}[3]{
\begin{figure}[ht]
\centering
\includegraphics[width=#2\textwidth]{fig/#1}
\caption{#3}
\label{fig:#1}
\end{figure}
}

\newcommand{\hide}[1]{}

\newtheorem{prop}{Proposition}

\usepackage{subcaption}

\title{Multi-Agent Actor-Critic for Mixed Cooperative-Competitive Environments} 

\author{Ryan Lowe$^*$\\ McGill University\\ OpenAI \And \hspace{-4mm}Yi Wu\thanks{Equal contribution.}\\ \hspace{-4mm}UC Berkeley \And Aviv Tamar\\ UC Berkeley \AND Jean Harb\\McGill University\\ OpenAI \And Pieter Abbeel\\UC Berkeley\\ OpenAI \And Igor Mordatch\\ OpenAI
}

\begin{document}

\maketitle

\begin{abstract}
We explore deep reinforcement learning methods for multi-agent domains. We begin by analyzing the difficulty of traditional algorithms in the multi-agent case: Q-learning is challenged by an inherent non-stationarity of the environment, while policy gradient suffers from a variance that increases as the number of agents grows. 
We then present an adaptation of actor-critic methods that considers action policies of other agents and is able to successfully learn policies that require complex multi-agent coordination. Additionally, we introduce a training regimen utilizing an ensemble of policies for each agent that leads to more robust multi-agent policies. We show the strength of our approach compared to existing methods in cooperative as well as competitive scenarios, where agent populations are able to discover various physical and informational coordination strategies.



\end{abstract}

\section{Introduction}
Reinforcement learning (RL) has recently been applied to solve challenging problems, from game playing \cite{mnih2015human,alphago} to robotics \cite{levine2015end}. In industrial applications, RL is emerging as a practical component in large scale systems such as data center cooling \cite{googleblog}. 
Most of the successes of RL have been in single agent domains, where modelling or predicting the behaviour of other actors in the environment is largely unnecessary.

However, there are a number of important applications that involve interaction between multiple agents, where emergent behavior and complexity arise from agents co-evolving together. For example, multi-robot control \cite{matignon12coordinated}, the discovery of communication and language \cite{sukhbaatar2016learning,foerster16b,mordatch2017emergence}, multiplayer games \cite{peng17starcraft}, and the analysis of social dilemmas \cite{multiagent_ssd} all operate in a multi-agent domain. Related problems, such as variants of hierarchical reinforcement learning \cite{dayan93feudal} can also be seen as a multi-agent system, with multiple levels of hierarchy being equivalent to multiple agents. Additionally, multi-agent self-play has recently been shown to be a useful training paradigm \cite{alphago, sukhbaatar2017intrinsic}. 
Successfully scaling RL to environments with multiple agents is crucial to building artificially intelligent systems that can productively interact with humans and each other.



Unfortunately, traditional reinforcement learning approaches such as Q-Learning or policy gradient are poorly suited to multi-agent environments. One issue is that each agent’s policy is changing as training progresses, and the environment becomes non-stationary from the perspective of any individual agent (in a way that is not explainable by changes in the agent's own policy). This presents learning stability challenges and prevents the straightforward use of past experience replay, which is crucial for stabilizing deep Q-learning. Policy gradient methods, on the other hand, usually exhibit very high variance when coordination of multiple agents is required. Alternatively, one can use model-based policy optimization which can learn optimal policies via back-propagation, but this requires a (differentiable) model of the world dynamics and assumptions about the interactions between agents. 
Applying these methods to competitive environments is also challenging from an optimization perspective, as evidenced by the notorious instability of adversarial training methods~\cite{goodfellow2014generative}.

In this work, we propose a general-purpose multi-agent learning algorithm that:
(1) leads to learned policies that only use local information (i.e.\@ their own observations) at execution time, (2) does not assume a differentiable model of the environment dynamics or any particular structure on the communication method between agents, and (3) is applicable not only to cooperative interaction but to competitive or mixed interaction involving both physical and communicative behavior. The ability to act in mixed cooperative-competitive environments may be critical for intelligent agents; while competitive training provides a natural curriculum for learning~\cite{sukhbaatar2017intrinsic}, agents must also exhibit cooperative behavior (e.g.\@ with humans) at execution time.

We adopt the framework of centralized training with decentralized execution, allowing the policies to use extra information to ease training, so long as this information is not used at test time. It is unnatural to do this with Q-learning without making additional assumptions about the structure of the environment, as the Q function generally cannot contain different information at training and test time. Thus, we propose a simple extension of actor-critic policy gradient methods where the critic is augmented with extra information about the policies of other agents, while the actor only has access to local information. After training is completed, only the local actors are used at execution phase, acting in a decentralized manner and equally applicable in cooperative and competitive settings.

Since the centralized critic function explicitly uses the decision-making policies of other agents, we additionally show that agents can learn approximate models of other agents online and effectively use them in their own policy learning procedure. We also introduce a method to improve the stability of multi-agent policies by training agents with an ensemble of policies, thus requiring robust interaction with a variety of collaborator and competitor policies. 
We empirically show the success of our approach compared to existing methods in cooperative as well as competitive scenarios, where agent populations are able to discover complex physical and communicative coordination strategies.

\section{Related Work}
\vspace{-2mm}

The simplest approach to learning in multi-agent settings is to use independently learning agents. This was attempted with Q-learning in \cite{tan93multi}, but does not perform well in practice \cite{matignon12independent}. As we will show, independently-learning policy gradient methods also perform poorly. 
One issue is that each agent's policy changes during training, resulting in a non-stationary environment and preventing the na{\"i}ve application of experience replay. 
Previous work has attempted to address this by inputting other agent’s policy parameters to the Q function \cite{hyper_q}, explicitly adding the iteration index to the replay buffer, or using importance sampling \cite{foerster_nonstat}. Deep Q-learning approaches have previously been investigated in \cite{tampuu2017multiagent} to train competing Pong agents.

The nature of interaction between agents can either be cooperative, competitive, or both and many algorithms are designed only for a particular nature of interaction. Most studied are cooperative settings, with strategies such as optimistic and hysteretic Q function updates \cite{lauer00distributed,hyst07,hyst17}, which assume that the actions of other agents are made to improve collective reward. Another approach is to indirectly arrive at cooperation via sharing of policy parameters \cite{gupta17cooperative}, but this requires homogeneous agent capabilities. These algorithms are generally not applicable in competitive or mixed settings. See \cite{panait05,busoniu2008comprehensive} for surveys of multi-agent learning approaches and applications.

Concurrently to our work, \cite{foerster2017counterfactual} proposed a similar idea of using policy gradient methods with a centralized critic, and test their approach on a StarCraft micromanagement task. Their approach differs from ours in the following ways: (1) they learn a single centralized critic for all agents, whereas we learn a centralized critic for each agent, allowing for agents with differing reward functions including competitive scenarios, (2) we consider environments with explicit communication between agents, (3) they combine recurrent policies with feed-forward critics, whereas our experiments use feed-forward policies (although our methods are applicable to recurrent policies), (4) we learn continuous policies whereas they learn discrete policies.  

\nocite{thomas2011conjugate}

Recent work has focused on learning grounded cooperative communication protocols between agents to solve various tasks \cite{sukhbaatar2016learning,foerster16b,mordatch2017emergence}. However, these methods are usually only applicable when the communication between agents is carried out over a dedicated, differentiable communication channel. 



Our method requires explicitly modeling decision-making process of other agents. The importance of such modeling has been recognized by both reinforcement learning \cite{boutilier96,boutilier03} and cognitive science communities \cite{frank_rsa}. \cite{hu98} stressed the importance of being robust to the decision making process of other agents, as do others by building Bayesian models of decision making. We incorporate such robustness considerations by requiring that agents interact successfully with an ensemble of any possible policies of other agents, improving training stability and robustness of agents after training.

\section{Background}
\label{sec:background}

\paragraph{Markov Games}
In this work, we consider a multi-agent extension of Markov decision processes (MDPs) called partially observable Markov games \cite{littman1994markov}. A Markov game for $N$ agents is defined by a set of states $\mathcal{S}$ describing the possible configurations of all agents, a set of actions $\mathcal{A}_1,...,\mathcal{A}_N$ and a set of observations $\mathcal{O}_1,...,\mathcal{O}_N$ for each agent. To choose actions, each agent $i$ uses a stochastic policy $\pol_{\theta_i} : \mathcal{O}_i \times \mathcal{A}_i \mapsto [0,1]$, which produces the next state according to the state transition function $\mathcal{T} : \mathcal{S} \times \mathcal{A}_1 \times ... \times \mathcal{A}_N \mapsto \mathcal{S}$.\footnote{To minimize notation we will often omit $\theta$ from the subscript of $\pol$.} 
Each agent $i$ obtains rewards as a function of the state and agent's action $r_i : \mathcal{S} \times \mathcal{A}_i \mapsto \mathbb{R}$, and receives a private observation correlated with the state $\mathbf{o}_i : \mathcal{S} \mapsto \mathcal{O}_i$. The initial states are determined by a distribution $\rho : \mathcal{S} \mapsto [0,1]$. Each agent $i$ aims to maximize its own total expected return $R_i = \sum_{t=0}^T \gamma^t r^t_i$ where $\gamma$ is a discount factor and $T$ is the time horizon.


\paragraph{Q-Learning and Deep Q-Networks (DQN).}
Q-Learning and DQN \cite{mnih2015human} are popular methods in reinforcement learning and have been previously applied to multi-agent settings \cite{foerster16b,hyper_q}. Q-Learning makes use of an action-value function for policy $\pol$ as $Q^{\pol}(s, a) = \mathbb{E}[R | s^t = s, a^t = a]$.
This Q function can be recursively rewritten as $Q^{\pol}(s, a) = \mathbb{E}_{s'}[r(s,a) + \gamma \mathbb{E}_{a' \sim \pol}[Q^{\pol}(s', a')]]$. DQN learns the action-value function $Q^*$ corresponding to the optimal policy by minimizing the loss:
\begin{equation}
    \mathcal{L}(\theta) = \mathbb{E}_{s,a,r,s'}[(Q^*(s,a|\theta) - y)^2], \qquad
    \text{where~ } \quad y = r + \gamma \max_{a'} \bar{Q}^*(s', a'),
\end{equation}
where $\bar{Q}$ is a target Q function whose parameters are periodically updated with the most recent $\theta$, which helps stabilize learning. Another crucial component of stabilizing DQN is the use of an experience replay buffer $\mathcal{D}$ containing tuples $(s,a,r,s')$.

Q-Learning can be directly applied to multi-agent settings by having each agent $i$ learn an independently optimal function $Q_i$ \cite{tan93multi}. However, because agents are independently updating their policies as learning progresses, the environment appears non-stationary from the view of any one agent, violating Markov assumptions required for convergence of Q-learning. Another difficulty observed in \cite{foerster_nonstat} is that the experience replay buffer cannot be used in such a setting since in general, $P(s'|s,a,\pol_1,...,\pol_N) \neq P(s'|s,a,\pol'_1,...,\pol'_N)$ when any $\pol_i \neq \pol_i'$.

\paragraph{Policy Gradient (PG) Algorithms.} Policy gradient methods are another popular choice for a variety of RL tasks.
The main idea is to directly adjust the parameters $\theta$ of the policy in order to maximize the objective $J(\theta) = \mathbb{E}_{s \sim p^{\pol}, a \sim {\pol}_\theta}[R]$ by taking steps in the direction of $\nabla_\theta J(\theta)$. Using the Q function defined previously, the gradient of the policy can be written as \cite{sutton2000policy}:
\begin{equation}
\nabla_\theta J(\theta) = \mathbb{E}_{s \sim p^{\pol}, a \sim {\pol}_\theta} [\nabla_\theta \log \pol_\theta(a|s) Q^{\pol} (s,a)],
\end{equation}
where $p^{\pol}$ is the state distribution. 
The policy gradient theorem has given rise to several practical algorithms, which often differ in how they estimate $Q^{\pol}$. For example, one can simply use a sample return $R^t = \sum_{i=t}^T \gamma^{i-t}r_i$, which leads to the REINFORCE algorithm \cite{williams1992simple}. 
Alternatively, one could learn an approximation of the true action-value function $Q^{\pol}(s,a)$ by e.g.\@ temporal-difference learning \cite{sutton1998reinforcement}; 
this $Q^{\pol}(s,a)$ is called the \textit{critic} and leads to a variety of \textit{actor-critic} algorithms \cite{sutton1998reinforcement}.

Policy gradient methods are known to exhibit high variance gradient estimates.  This is exacerbated in multi-agent settings; since an agent's reward usually depends on the actions of many agents, the reward conditioned only on the agent's own actions (when the actions of other agents are not considered in the agent's optimization process) exhibits much more variability, thereby increasing the variance of its gradients.
Below, we show a simple setting where the probability of taking a gradient step in the correct direction decreases exponentially with the number of agents. 
\begin{prop}\label{prop:pg}
Consider $N$ agents with binary actions: $P(a_i=1) = \theta_i$, where $R(a_1,\dots,a_N) = \mathbf{1}_{a_1=\dots=a_N}$. We assume an uninformed scenario, in which agents are initialized to $\theta_i=0.5 \ \forall i$. Then, if we are estimating the gradient of the cost $J$ with policy gradient, we have:

$$ P(\langle \hat{\nabla} J, \nabla J \rangle > 0) \propto (0.5)^N
$$

where $\hat{\nabla} J$ is the policy gradient estimator from a single sample, and $\nabla J$ is the true gradient.
\end{prop}

\begin{proof}
See Appendix.
\end{proof}
The use of baselines, such as value function baselines typically used to ameliorate high variance, is problematic in multi-agent settings due to the non-stationarity issues mentioned previously.

\paragraph{Deterministic Policy Gradient (DPG) Algorithms.} It is also possible to extend the policy gradient framework to deterministic policies $\cpol_\theta: \mathcal{S} \mapsto \mathcal{A}$ \cite{silver2014deterministic}. In particular, under certain conditions we can write the gradient of the objective $J(\theta) = \mathbb{E}_{s \sim p^{\cpol}}[R(s,a)]$ as:
\begin{equation}
\nabla_\theta J(\theta) = \mathbb{E}_{s \sim \mathcal{D}} [\nabla_\theta \cpol_\theta(a|s) \nabla_a Q^{\cpol} (s,a)|_{a=\cpol_\theta (s)}]
\end{equation}
Since this theorem relies on $\nabla_a Q^{\cpol} (s,a)$, it requires that the action space $\mathcal{A}$ (and thus the policy $\cpol$) be continuous. 

\textit{Deep deterministic policy gradient} (DDPG) \cite{lillicrap2015continuous} is a variant of DPG where the policy $\cpol$ and critic $Q^{\cpol}$ are approximated with deep neural networks. DDPG is an off-policy algorithm, and samples trajectories from a replay buffer of experiences that are stored throughout training. DDPG also makes use of a target network, as in DQN \cite{mnih2015human}.



\section{Methods}
\label{sec:methods}
\subsection{Multi-Agent Actor Critic}\label{sec:maac}

\begin{wrapfigure}{r}{0.45\textwidth}
\vspace{-20mm}
\includegraphics[width=.9\linewidth]{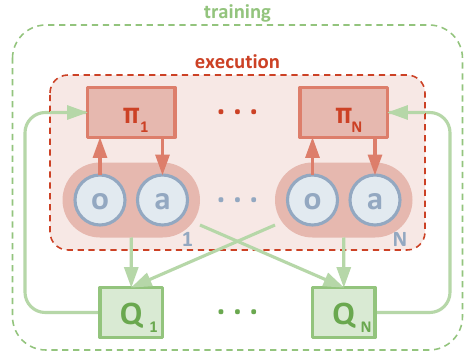}
\caption{\label{fig:model} Overview of our multi-agent decentralized actor, centralized critic approach.\vspace{-2mm}}
\end{wrapfigure}

We have argued in the previous section 
that na{\"i}ve policy gradient methods perform poorly in simple multi-agent settings, and this is supported in our experiments in Section \ref{sec:experiments}. Our goal in this section is to derive an algorithm that works well 
in such settings. However, we would like to operate under the following constraints: (1) the learned policies can only use local information (i.e.\@ their own observations) at execution time, (2) we do not assume a differentiable model of the environment dynamics, unlike in \cite{mordatch2017emergence}, 
and (3) we do not assume any particular structure on the communication method between agents (that is, we don't assume a differentiable communication channel). Fulfilling the above desiderata would provide a general-purpose multi-agent learning algorithm that could be applied not just to cooperative games with explicit communication channels, but competitive games and games involving only physical interactions between agents.


Similarly to \cite{foerster16b}, we accomplish our goal by adopting the framework of centralized training with decentralized execution. Thus, we allow the policies to use extra information to ease training, so long as this information is not used at test time. It is unnatural to do this with Q-learning, as the Q function generally cannot contain different information at training and test time. Thus, we propose a simple extension of actor-critic policy gradient methods where the critic is augmented with extra information about the policies of other agents. 

More concretely, consider a game with $N$ agents with policies parameterized by $\pmb{\theta} = \{\theta_1, ..., \theta_N\}$, and let $\pol = \{\pol_1, ..., \pol_N\}$ be the set of all agent policies. Then we can write the gradient of the expected return for agent $i$, $J(\theta_i) = \mathbb{E}[R_i]$ as:
\begin{equation}
\nabla_{\theta_i} J(\theta_i) = \mathbb{E}_{s\sim p^{\cpol},a_i \sim \pol_i} [\nabla_{\theta_i} \log \pol_i(a_i|o_i) Q^{\pol}_i (\mathbf{x},a_1, ..., a_N)]. 
\end{equation}

Here $Q^{\pol}_i (\mathbf{x},a_1, ..., a_N)$ is a \textit{centralized action-value function} that takes as input the actions of all agents, $a_1,\ldots, a_N$, in addition to some state information $\mathbf{x}$, and outputs the Q-value for agent $i$. In the simplest case, $\mathbf{x}$ could consist of the observations of all agents, $\mathbf{x} = (o_1, ..., o_N)$, however we could also include additional state information if available. Since each $Q^{\pol}_i$ is learned separately, agents can have arbitrary reward structures, including conflicting rewards in a competitive setting.

We can extend the above idea to work with deterministic policies. If we now consider $N$ continuous policies $\cpol_{\theta_i}$ w.r.t. parameters $\theta_i$ (abbreviated as $\cpol_i$), the gradient can be written as:
\begin{equation}
\nabla_{\theta_i} J(\cpol_i) = \mathbb{E}_{\mathbf{x},a \sim \mathcal{D}}[\nabla_{\theta_i} \cpol_i(a_i|o_i) \nabla_{a_i} Q^{\cpol}_i (\mathbf{x},a_1, ..., a_N)|_{a_i=\cpol_i (o_i)}],
\end{equation}
Here the experience replay buffer $\mathcal{D}$ contains the tuples $(\mathbf{x},\mathbf{x}',a_1,\ldots,a_N,r_1,\ldots,r_N)$, recording experiences of all agents. 
The centralized action-value function $Q^{\cpol}_i$ is updated as:
\begin{equation}\label{eq:q_func}
\mathcal{L}(\theta_i) = \mathbb{E}_{\mathbf{x},a,r,\mathbf{x}'}[(Q^{\cpol}_i(\mathbf{x},a_1,\ldots,a_N) - y)^2], \;\;\;\;
y = r_i + \gamma\, Q^{\cpol'}_i(\mathbf{x}', a_1',\ldots,a_N')\big|_{a_j'=\cpol'_j(o_j)},
\end{equation}
where $\cpol' = \{\cpol_{\theta'_1}, ..., \cpol_{\theta'_N} \}$ is the set of target policies with delayed parameters $\theta'_i$. As shown in Section \ref{sec:experiments}, we find the centralized critic with deterministic policies works very well in practice, and refer to it as \textit{multi-agent deep deterministic policy gradient} (MADDPG). 
We provide the description of the full algorithm in the Appendix. 

A primary motivation behind MADDPG is that, if we know the actions taken by all agents, the environment is stationary even as the policies change, since $P(s'|s,a_1,...,a_N,\pol_1,...,\pol_N) = P(s'|s,a_1,...,a_N) = P(s'|s,a_1,...,a_N,\pol'_1,...,\pol'_N)$ for any $\pol_i \neq \pol_i'$. This is not the case if we do not explicitly condition on the actions of other agents, as done for most traditional RL methods. 

Note that we require the policies of other agents to apply an update in Eq.~\ref{eq:q_func}. Knowing the observations and policies of other agents is not a particularly restrictive assumption; if our goal is to train agents to exhibit complex communicative behaviour in simulation, this information is often available to all agents.
However, we can relax this assumption if necessary by learning the policies of other agents from observations --- we describe a method of doing this in Section \ref{sec:modellearning}.

\subsection{Inferring Policies of Other Agents}
\label{sec:modellearning}
To remove the assumption of knowing other agents' policies, as required in Eq.~\ref{eq:q_func}, each agent $i$ can additionally maintain an approximation $\hat{\cpol}_{\phi_i^j}$ (where $\phi$ are the parameters of the approximation; henceforth $\hat{\cpol}_i^j$) to the true policy of agent $j$, $\cpol_j$. This approximate policy is learned by maximizing the log probability of agent $j$'s actions, with an entropy regularizer:
\begin{equation}\label{eq:loss-modellearn}
\mathcal{L}(\phi_i^j)=-\mathbb{E}_{o_j,a_j}\left[\log \hat{\cpol}_i^j(a_j|o_j) + \lambda H(\hat{\cpol}_i^j) \right],
\end{equation}
where $H$ is the entropy of the policy distribution. With the approximate policies, $y$ in Eq.~\ref{eq:q_func} can be replaced by an approximate value $\hat{y}$ calculated as follows: 
\begin{equation}\label{eq:approx_q_func}
\hat{y} = r_i + \gamma Q^{\cpol'}_i(\mathbf{x}', \hat{\cpol}'^1_i(o_1),\ldots,\cpol'_i(o_i),\ldots,\hat{\cpol}'^N_i(o_N)),
\end{equation}
where $\hat{\cpol}'^j_i$ denotes the target network for the approximate policy $\hat{\cpol}_i^j$.
Note that Eq.~\ref{eq:loss-modellearn} can be optimized in a completely online fashion: before updating  $Q^{\cpol}_i$, the centralized Q function, we take the latest samples of each agent $j$ from the replay buffer to perform a single gradient step to update $\phi_i^j$. 
Note also that, in the above equation, we input the action log probabilities of each agent directly into $Q$, rather than sampling.

\subsection{Agents with Policy Ensembles}
\label{sec:ensemble}
As previously mentioned, a recurring problem in multi-agent reinforcement learning is the environment non-stationarity due to the agents' changing policies. This is particularly true in competitive settings, where agents can derive a strong policy by overfitting to the behavior of their competitors. Such policies are undesirable as they are brittle and may fail when the competitors alter strategies.


To obtain multi-agent policies that are more robust to changes in the policy of competing agents, we propose to train a collection of $K$ different sub-policies. At each episode, we randomly select one particular sub-policy for each agent to execute. 
Suppose that policy $\cpol_i$ is an ensemble of $K$ different sub-policies with sub-policy $k$ denoted by $\cpol_{\theta_i^{(k)}}$ (denoted as $\cpol_i^{(k)}$). For agent $i$, we are then maximizing the ensemble objective:
$
J_e(\cpol_i)=\mathbb{E}_{k\sim\textrm{unif}(1,K),s\sim p^{\cpol},a\sim\cpol_i^{(k)}}\left[R_i(s,a)\right].
$

Since different sub-policies will be executed in different episodes, we maintain a replay buffer $\mathcal{D}_i^{(k)}$ for each sub-policy $\cpol_i^{(k)}$ of agent $i$. Accordingly, we can derive the gradient of the ensemble objective with respect to $\theta_i^{(k)}$ as follows:
\begin{equation}\label{eq:J_ensemble_grad}
\nabla_{\theta_i^{(k)}} J_e(\cpol_i)=\frac{1}{K}\mathbb{E}_{\mathbf{x},a\sim \mathcal{D}_i^{(k)}}\left[\nabla_{\theta_i^{(k)}} \cpol_i^{(k)}(a_i|o_i)\nabla_{a_i}Q^{\cpol_i}\left(\mathbf{x},a_1,\ldots,a_N\right)\Big|_{a_i=\cpol_i^{(k)}(o_i)}\right].
\end{equation}

\section{Experiments\protect\footnote[1]{
Videos of our experimental results can be viewed at \href{https://sites.google.com/site/multiagentac/}{https://sites.google.com/site/multiagentac/}}}
\label{sec:experiments}

\subsection{Environments}
\label{sec:environments}
To perform our experiments, we adopt the grounded communication environment proposed in \cite{mordatch2017emergence}\footnote{Code can be found here at https://github.com/openai/multiagent-particle-envs}, which consists of $N$ agents and $L$ landmarks inhabiting a two-dimensional world with continuous space and discrete time. Agents may take physical actions in the environment and communication actions that get broadcasted to other agents. Unlike  \cite{mordatch2017emergence}, we do not assume that all agents have identical action and observation spaces, or act according to the same policy $\pol$. We also consider games that are both cooperative (all agents must maximize a shared return) and competitive (agents have conflicting goals). Some environments require explicit communication between agents in order to achieve the best reward, while in other environments agents can only perform physical actions. We provide details for each environment below.

\hide{
\insertfigure{tasks}{1.00}{Illustrations of the experimental environment and some tasks we consider, including a) \emph{Cooperative Communication} b) \emph{Predator-Prey} c) \emph{Cooperative Navigation} d) \emph{Physical Deception}. See webpage for videos of all experimental results.}
}

\begin{figure}[ht]
\centering
\includegraphics[width=0.95\textwidth]{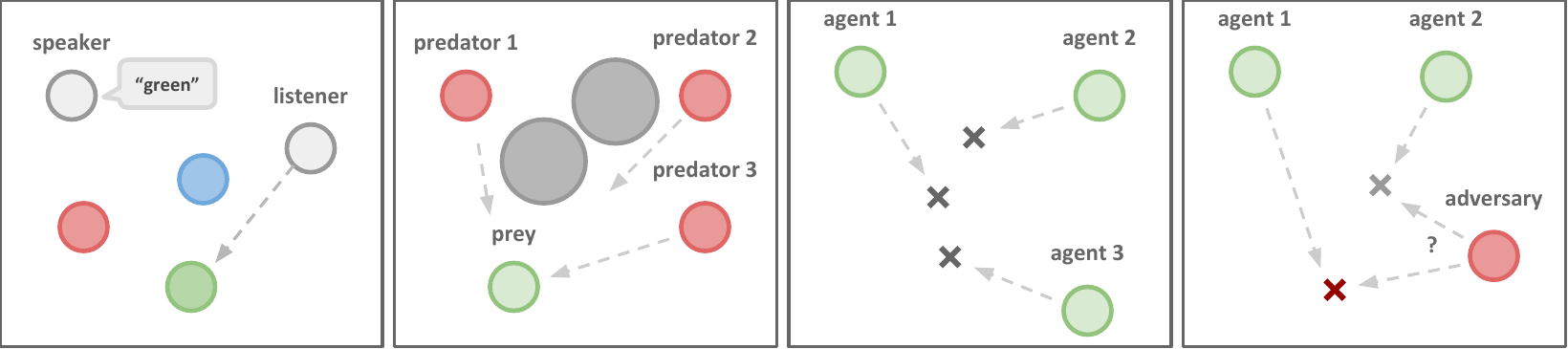}
\caption{Illustrations of the experimental environment and some tasks we consider, including a) \emph{Cooperative Communication} b) \emph{Predator-Prey} c) \emph{Cooperative Navigation} d) \emph{Physical Deception}. See webpage for videos of all experimental results.\vspace{-2mm}}
\label{fig:tasks}
\end{figure}



\textbf{Cooperative communication.} This task consists of two cooperative agents, a speaker and a listener, who are placed in an environment with three landmarks of differing colors. At each episode, the listener must navigate to a landmark of a particular color, and obtains reward based on its distance to the correct landmark. However, while the listener can observe the relative position and color of the landmarks, it does not know which landmark it must navigate to. Conversely, the speaker's observation consists of the correct landmark color, and it can produce a communication output at each time step which is observed by the listener. Thus, the speaker must learn to output the landmark colour based on the motions of the listener. 
Although this problem is relatively simple, as we show in Section \ref{sec:maddpgexperiments} it poses a significant challenge to traditional RL algorithms.

\textbf{Cooperative navigation.} In this environment, agents must cooperate through physical actions to reach a set of $L$ landmarks. Agents observe the relative positions of other agents and landmarks, and are collectively rewarded based on the proximity of any agent to each landmark. In other words, the agents have to `cover' all of the landmarks. Further, the agents occupy significant physical space and are penalized when colliding with each other. Our agents learn to infer the landmark they must cover, and move there while avoiding other agents. 

\textbf{Keep-away.} This scenario consists of $L$ landmarks including a target landmark, $N$ cooperating agents who know the target landmark and are rewarded based on their distance to the target, and $M$ \textit{adversarial} agents who must prevent the cooperating agents from reaching the target.  Adversaries accomplish this by physically pushing the agents away from the landmark, temporarily occupying it. While the adversaries are also rewarded based on their distance to the target landmark, they do not know the correct target; this must be inferred from the movements of the agents. 

\textbf{Physical deception.} Here, $N$ agents cooperate to reach a single target landmark from a total of $N$ landmarks. They are rewarded based on the minimum distance of any agent to the target (so only one agent needs to reach the target landmark). However, a lone adversary also desires to reach the target landmark; the catch is that the adversary does not know which of the landmarks is the correct one. Thus the cooperating agents, who are penalized based on the adversary distance to the target, learn to spread out and cover all landmarks so as to deceive the adversary.

\textbf{Predator-prey.} In this variant of the classic predator-prey game, $N$ slower cooperating agents must chase the faster adversary around a randomly generated environment with $L$ large landmarks impeding the way. Each time the cooperative agents collide with an adversary, the agents are rewarded while the adversary is penalized. Agents observe the relative positions and velocities of the agents, and the positions of the landmarks.

\textbf{Covert communication.} This is an adversarial communication environment, where a speaker agent (`Alice') must communicate a message to a listener agent (`Bob'), who must reconstruct the message at the other end. However, an adversarial agent (`Eve') is also observing the channel, and wants to reconstruct the message --- Alice and Bob are penalized based on Eve's reconstruction, and thus Alice must encode her message using a randomly generated \textit{key}, known only to Alice and Bob. This is similar to the cryptography environment considered in \cite{abadi2016learning}.


\begin{figure}
\begin{subfigure}{.52\textwidth}
  \centering
  \includegraphics[width=1\linewidth]{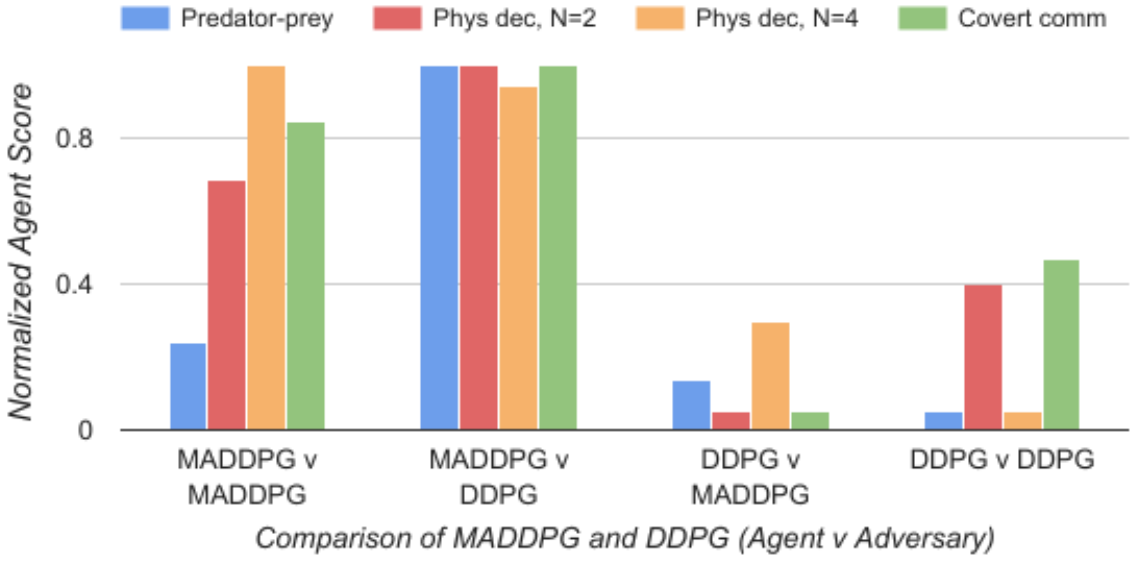}
\end{subfigure}
\hspace{1mm}
\begin{subfigure}{.48\textwidth}
  \centering
  \includegraphics[width=1\linewidth]{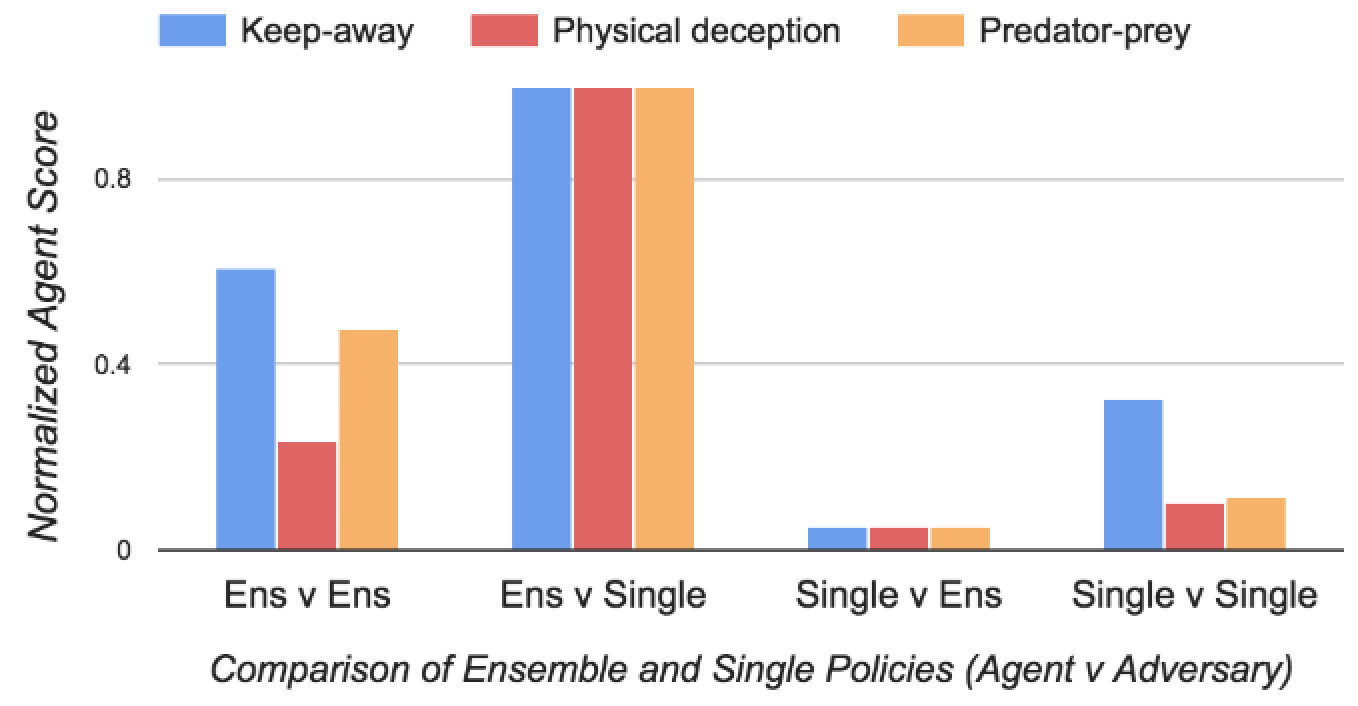}
\end{subfigure}
\vspace{-2mm}
\caption{Comparison between MADDPG and DDPG (left), and between single policy MADDPG and ensemble MADDPG (right) on the competitive environments. Each bar cluster shows the 0-1 normalized score for a set of competing policies (agent v adversary), where a higher score is better for the agent. In all cases, MADDPG outperforms DDPG when directly pitted against it, and similarly for the ensemble against the single MADDPG policies. Full results are given in the Appendix. \vspace{-3mm}}\label{fig:adv}
\end{figure}

\subsection{Comparison to Decentralized Reinforcement Learning Methods}
\label{sec:maddpgexperiments}

\begin{wrapfigure}{r}{0.5\textwidth}
\vspace{-6mm}
\includegraphics[width=1\linewidth]{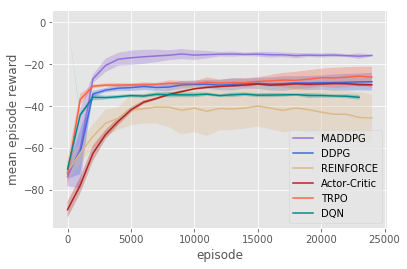}

\caption{\label{fig:comm_curves} Agent reward on cooperative communication after 25000 episodes. \vspace{-3mm}}
\end{wrapfigure}
We implement our MADDPG algorithm and evaluate it on the environments presented in Section \ref{sec:environments}. Unless otherwise specified, our policies are parameterized by a two-layer ReLU MLP with 64 units per layer. 
The messages sent between agents are soft approximations to discrete messages, calculated using the Gumbel-Softmax estimator \cite{jang2016categorical}. 
To evaluate the quality of policies learned in competitive settings, we pitch MADDPG agents against DDPG agents, and compare the resulting success of the agents and adversaries in the environment. We train our models until convergence, and then evaluate them by averaging various metrics for 1000 further iterations. 
We provide the tables and details of our results on all environments in the Appendix, and summarize them here. \hide{ (including hyperparameters)}

We first examine the cooperative communication scenario. Despite the simplicity of the task (the speaker only needs to learn to output its observation), traditional RL methods such as DQN, Actor-Critic, a first-order implementation of TRPO, and DDPG all fail to learn the correct behaviour (measured by whether the listener is within a short distance from the target landmark). In practice we observed that the listener learns to ignore the speaker and simply moves to the middle of all observed landmarks. We plot the learning curves over 25000 episodes for various approaches in Figure \ref{fig:comm_curves}.

\begin{figure}
\begin{subfigure}{.50\textwidth}
  \centering
  \rotatebox{90}{\scriptsize{Cooperative Comm.}}
  \includegraphics[width=0.31\linewidth]{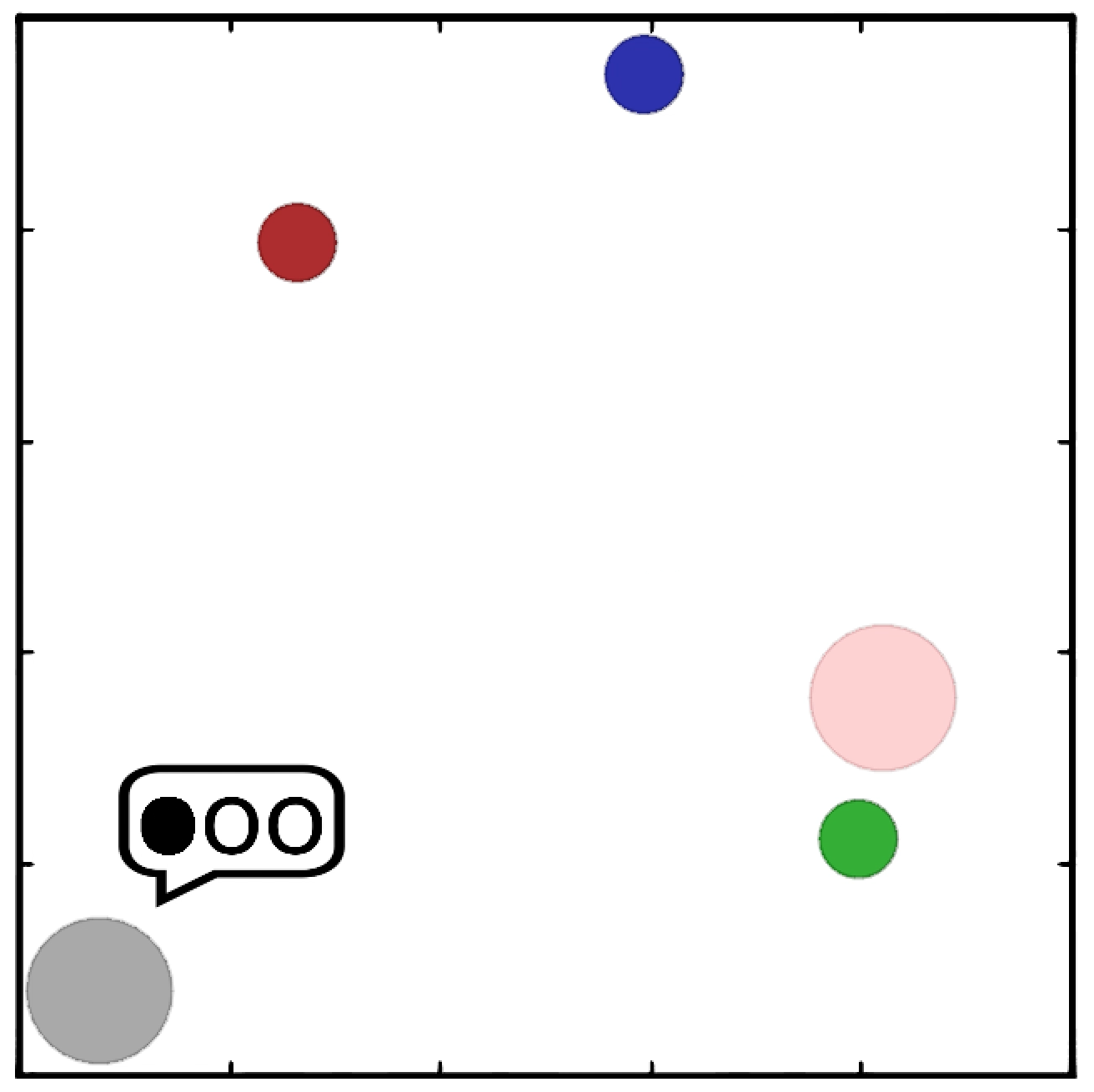}
  \includegraphics[width=0.31\linewidth]{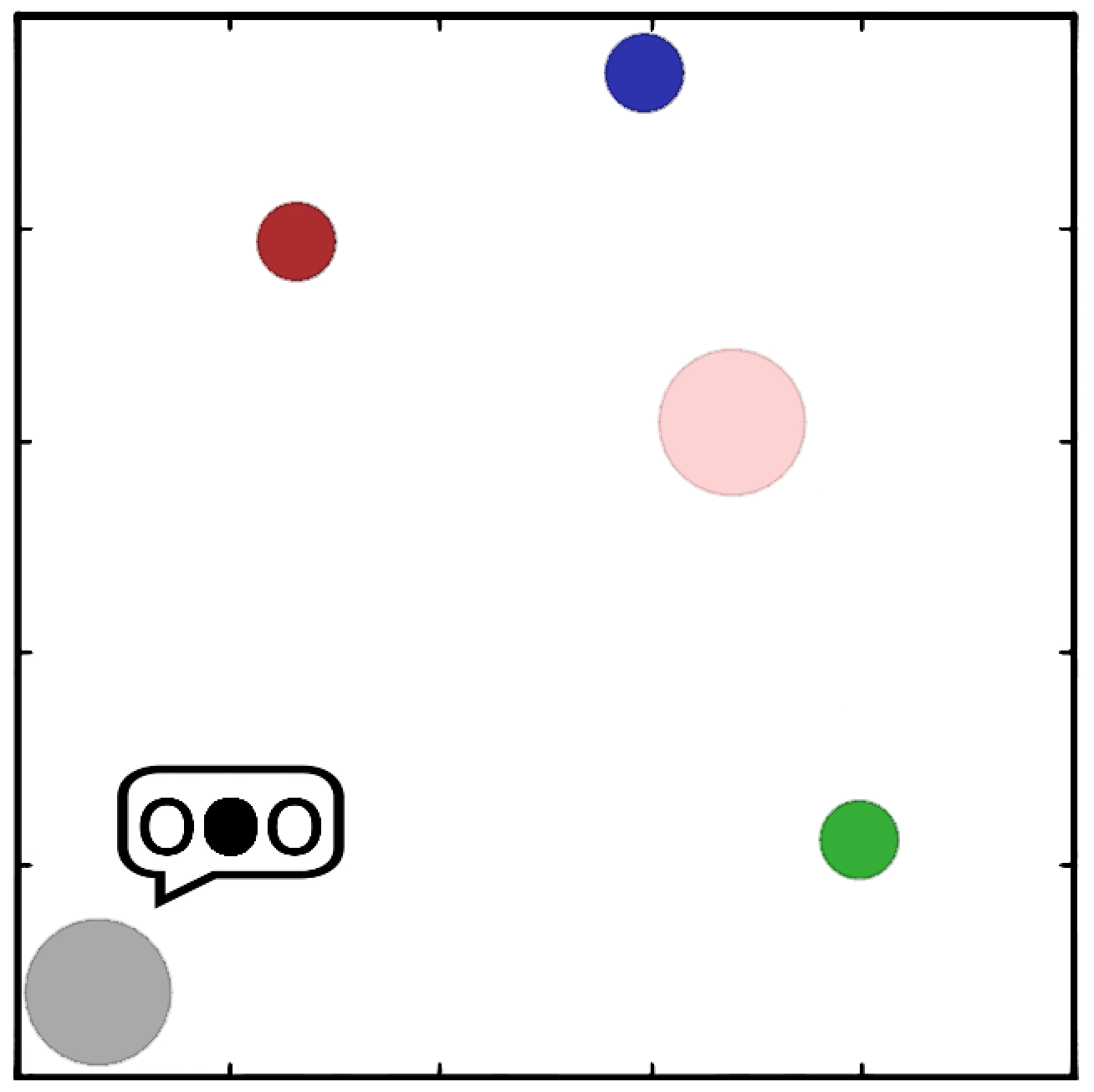}
  \includegraphics[width=0.31\linewidth]{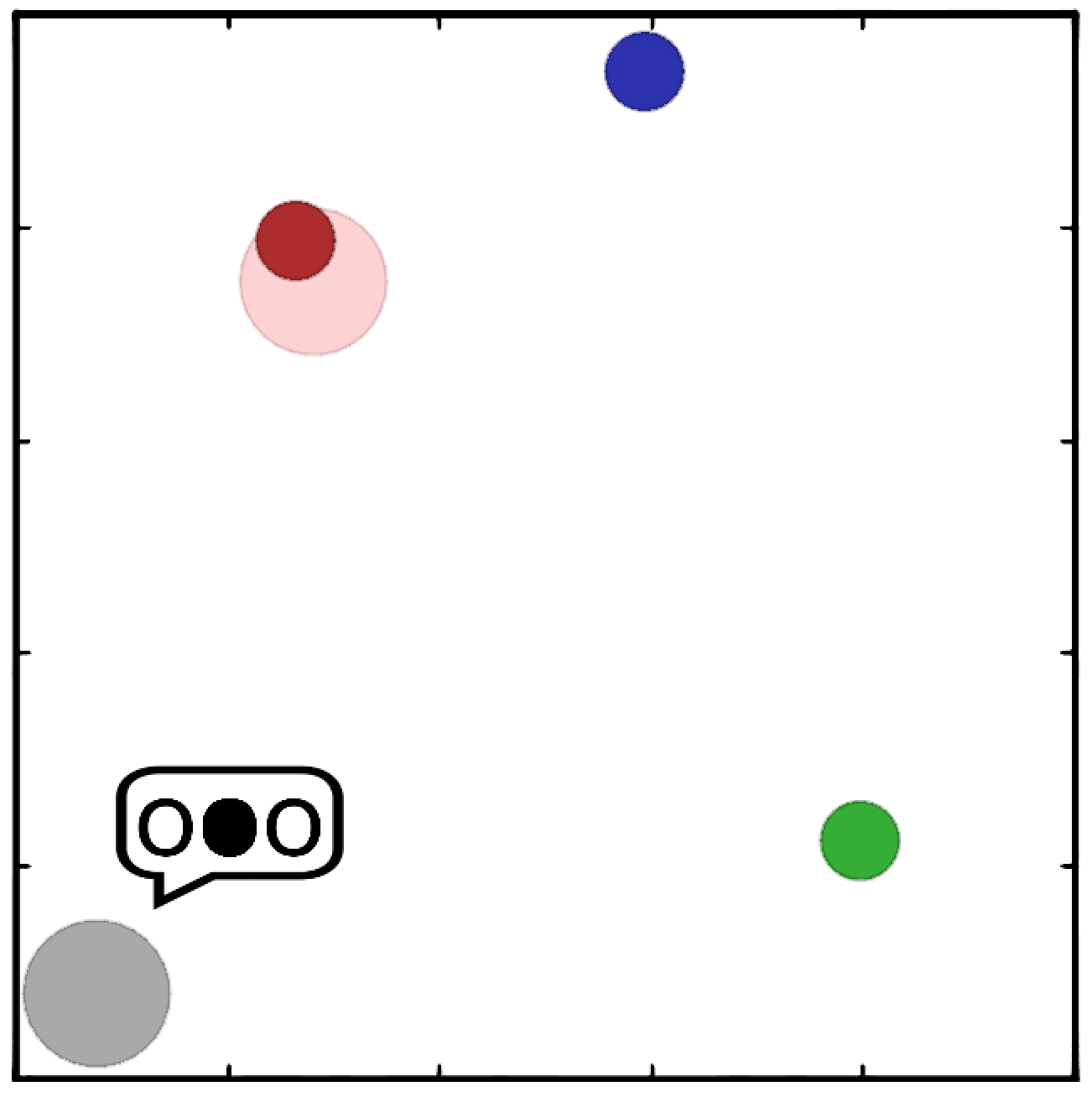}
  \rotatebox{90}{\scriptsize{Physical Deception}}
  \includegraphics[width=0.31\linewidth]{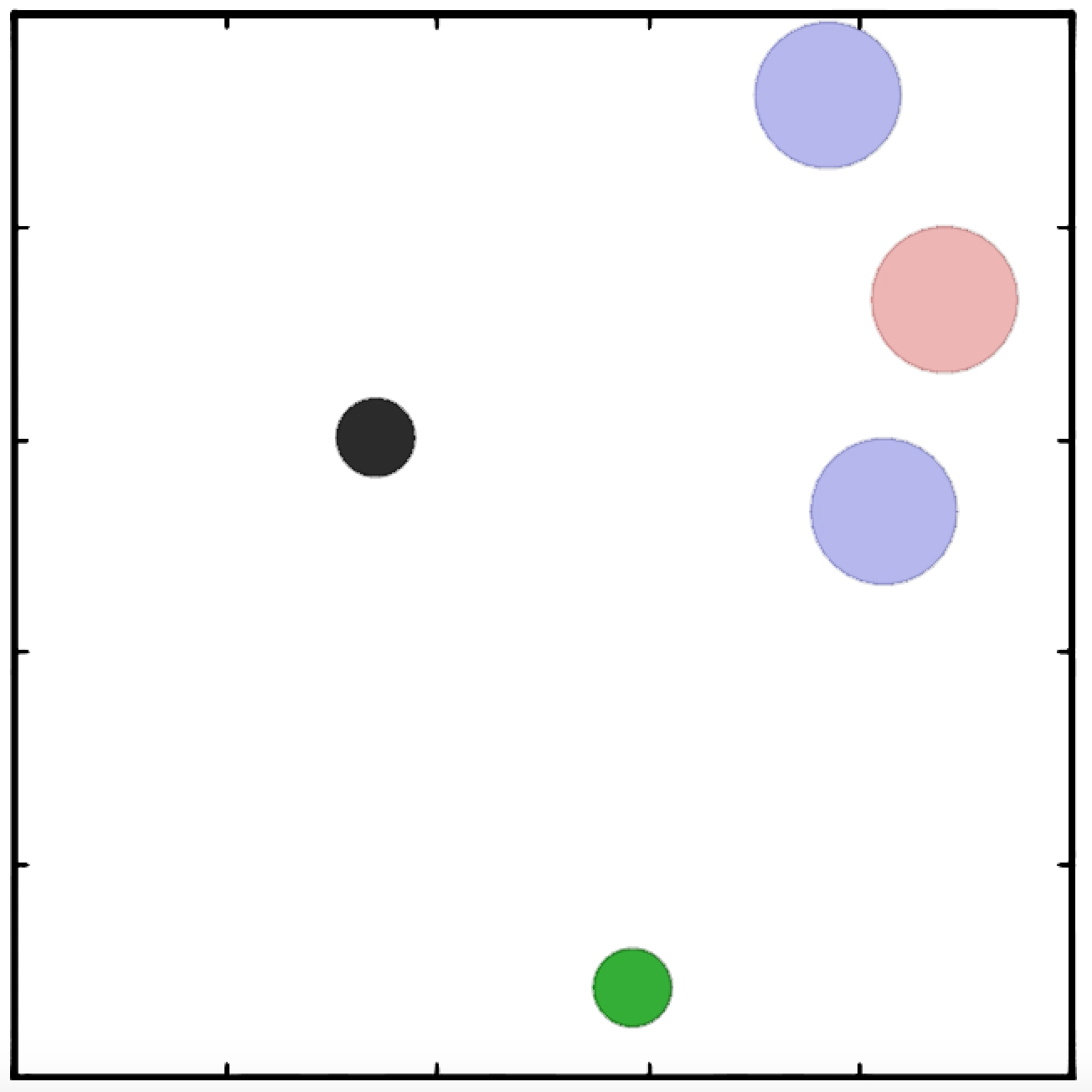}
  \includegraphics[width=0.31\linewidth]{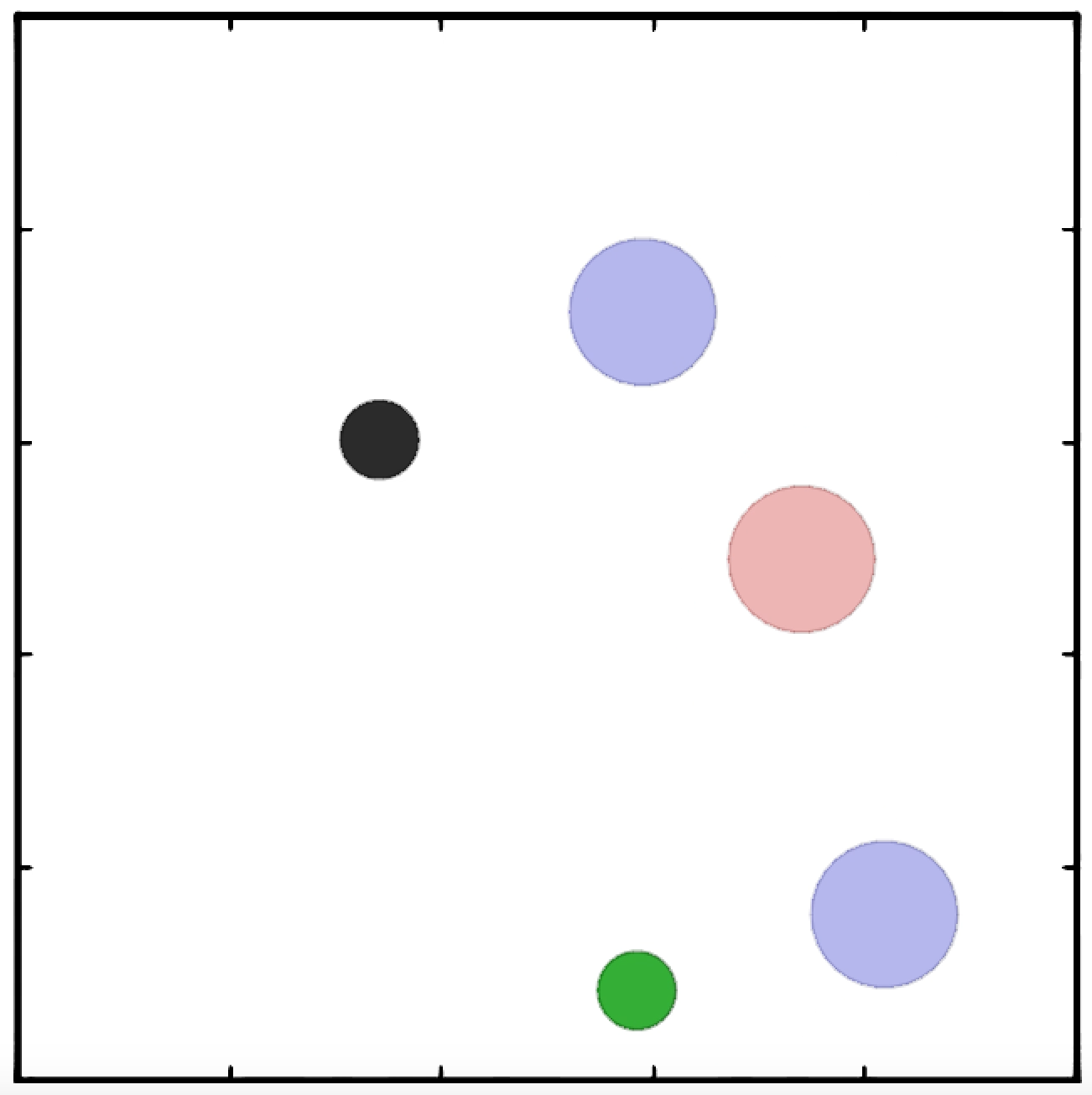}
  \includegraphics[width=0.31\linewidth]{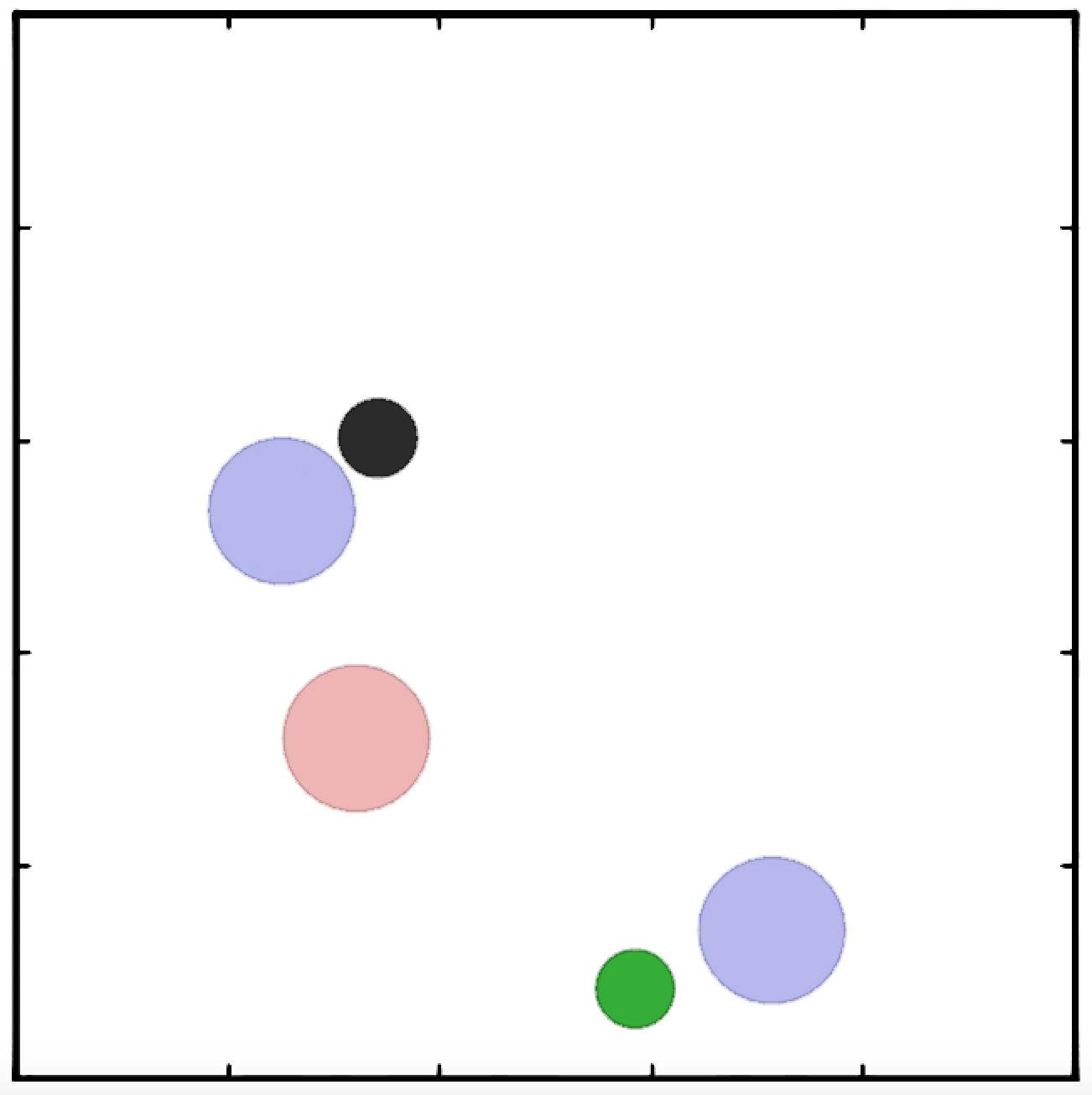}
  \caption{MADDPG}
\end{subfigure}
\hspace{1mm}
\begin{subfigure}{.50\textwidth}
  \centering
  \includegraphics[width=0.31\linewidth]{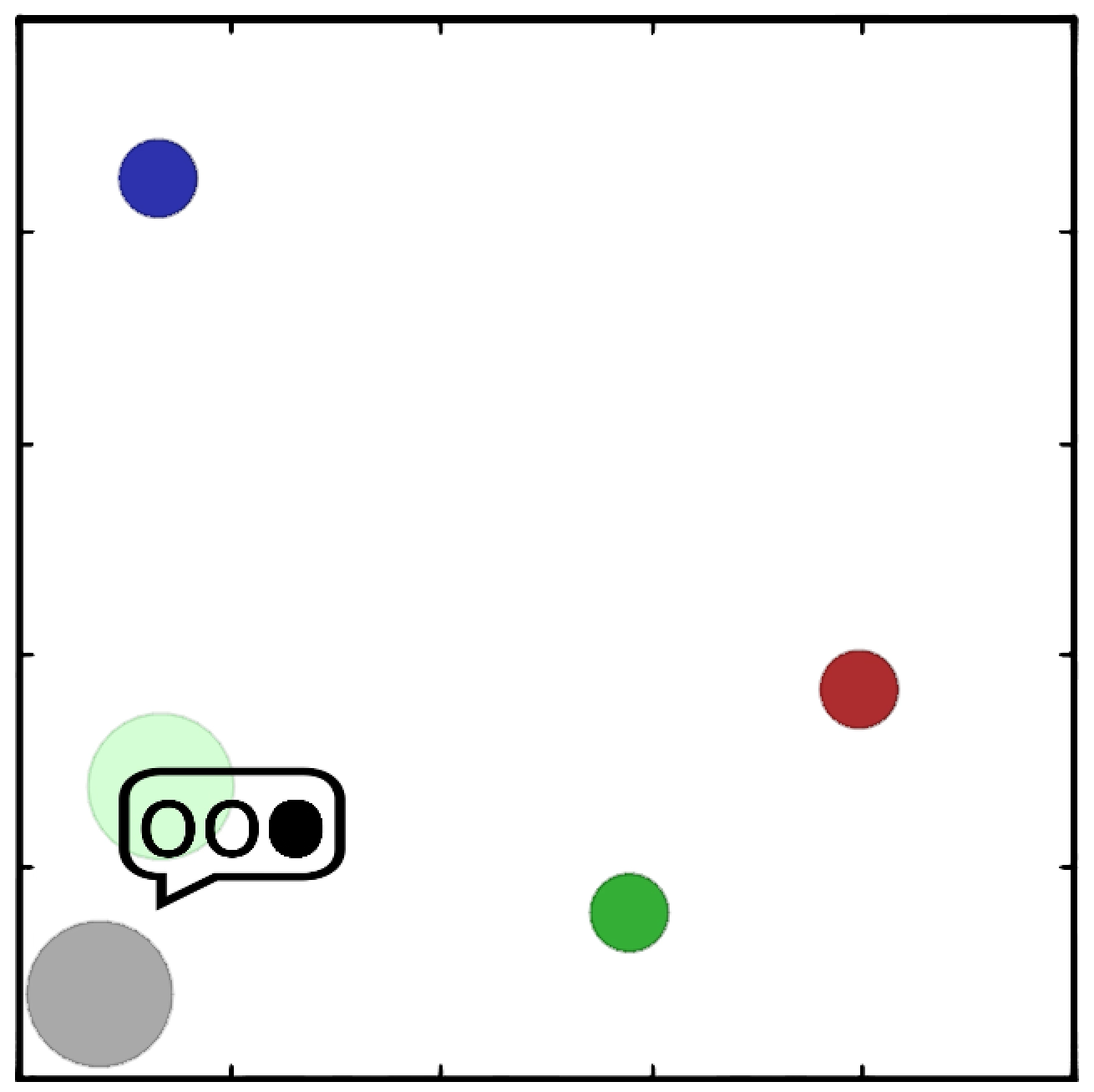}
  \includegraphics[width=0.31\linewidth]{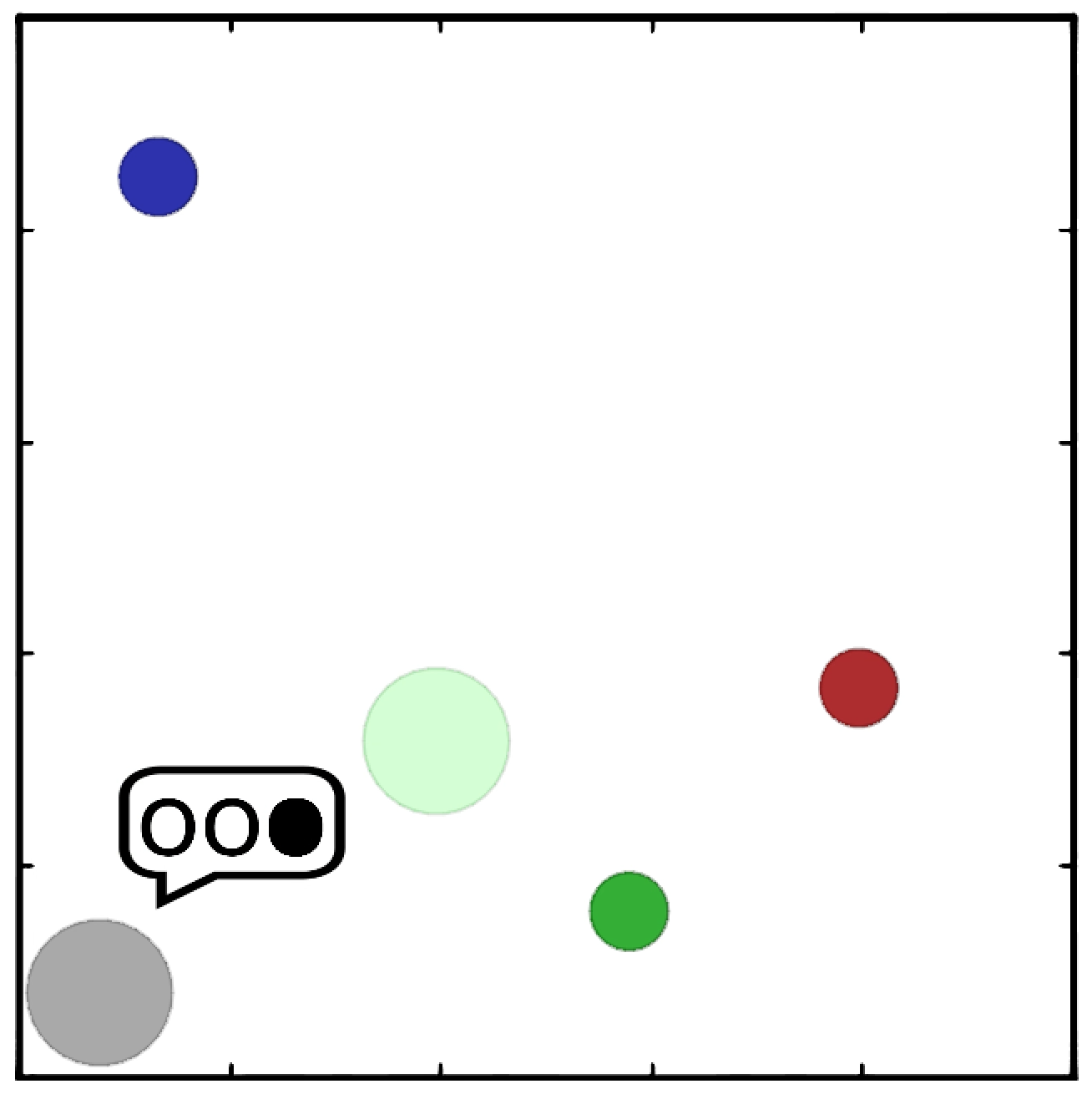}
  \includegraphics[width=0.31\linewidth]{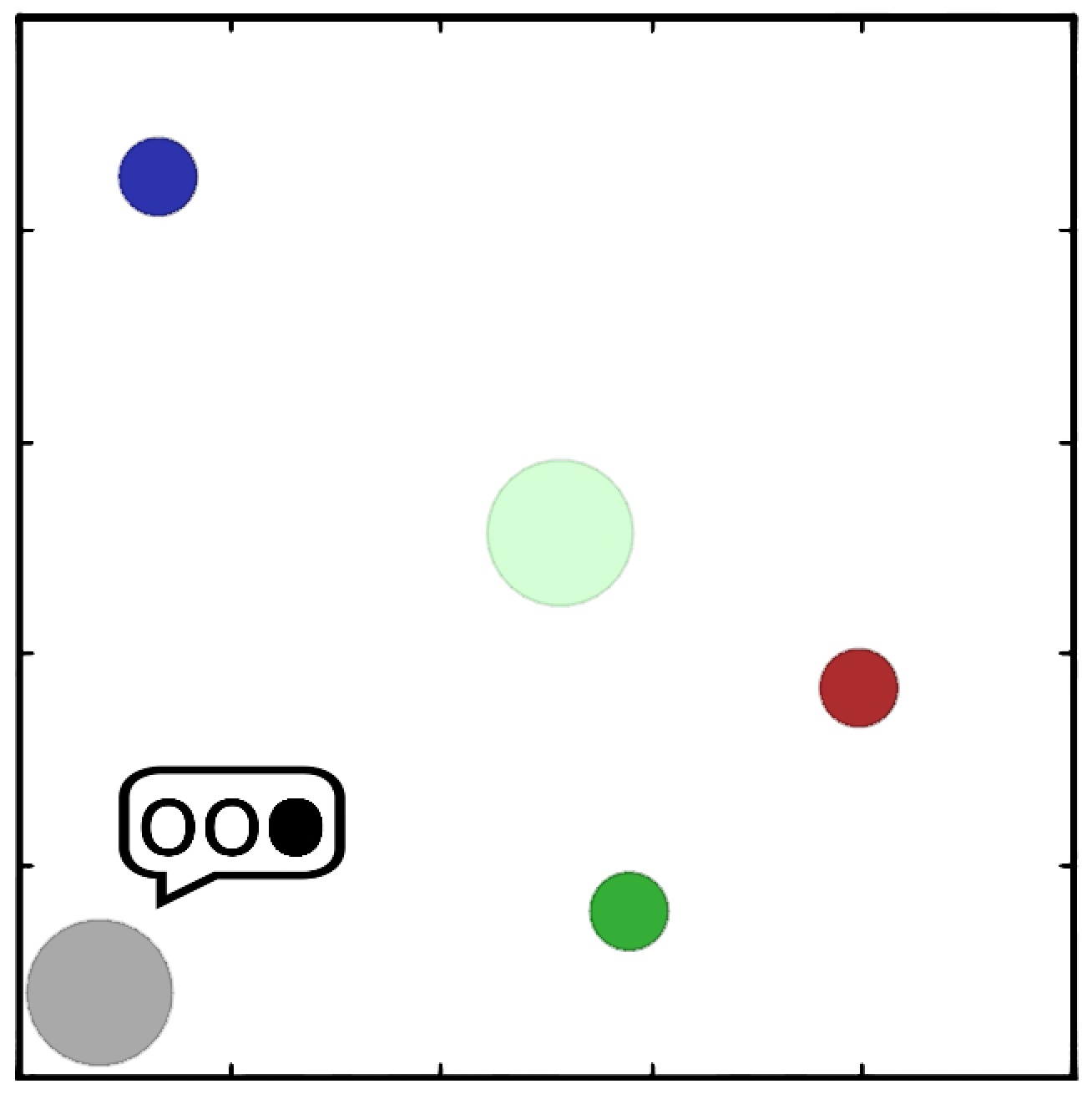}
  \includegraphics[width=0.31\linewidth]{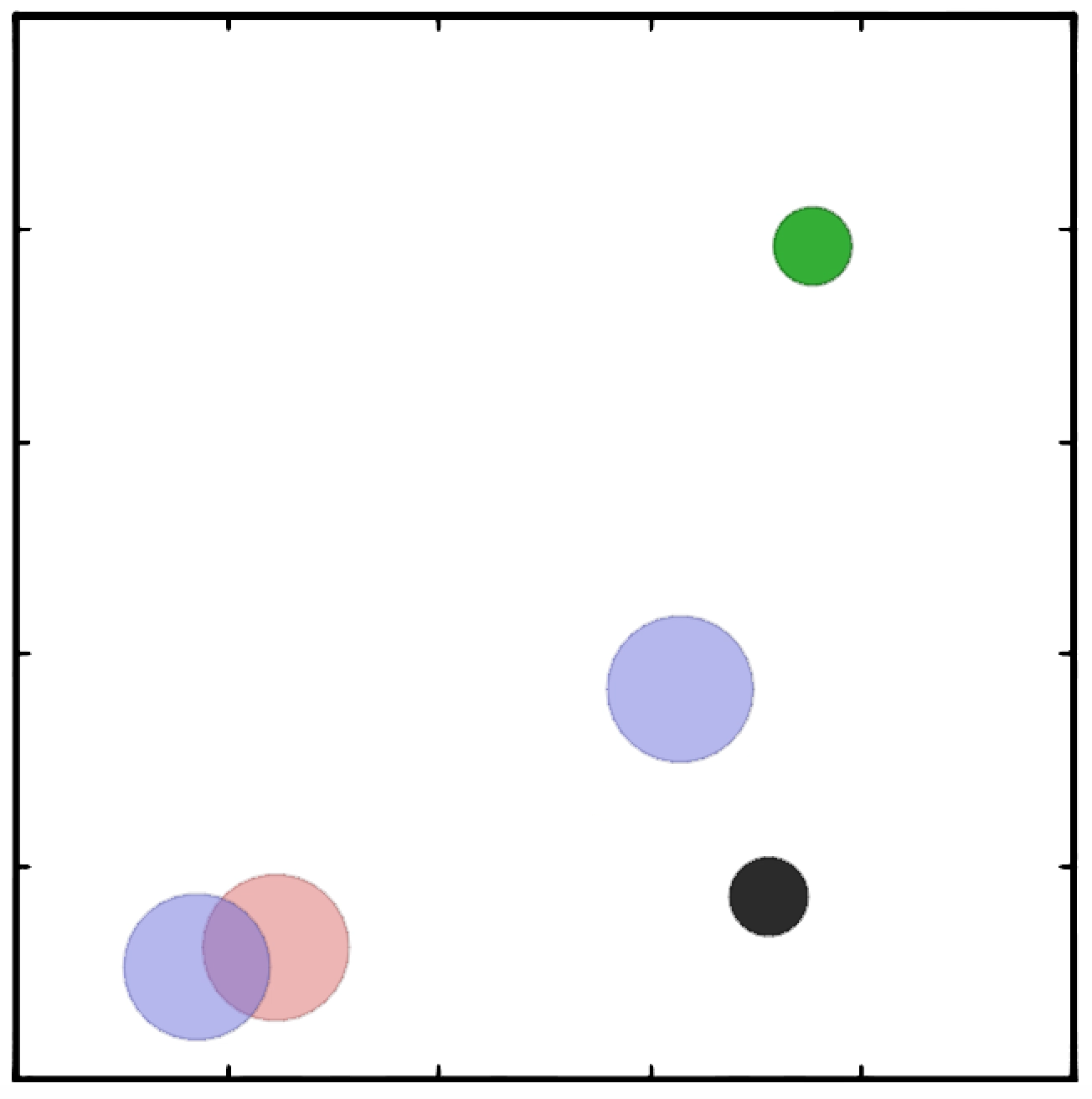}
  \includegraphics[width=0.31\linewidth]{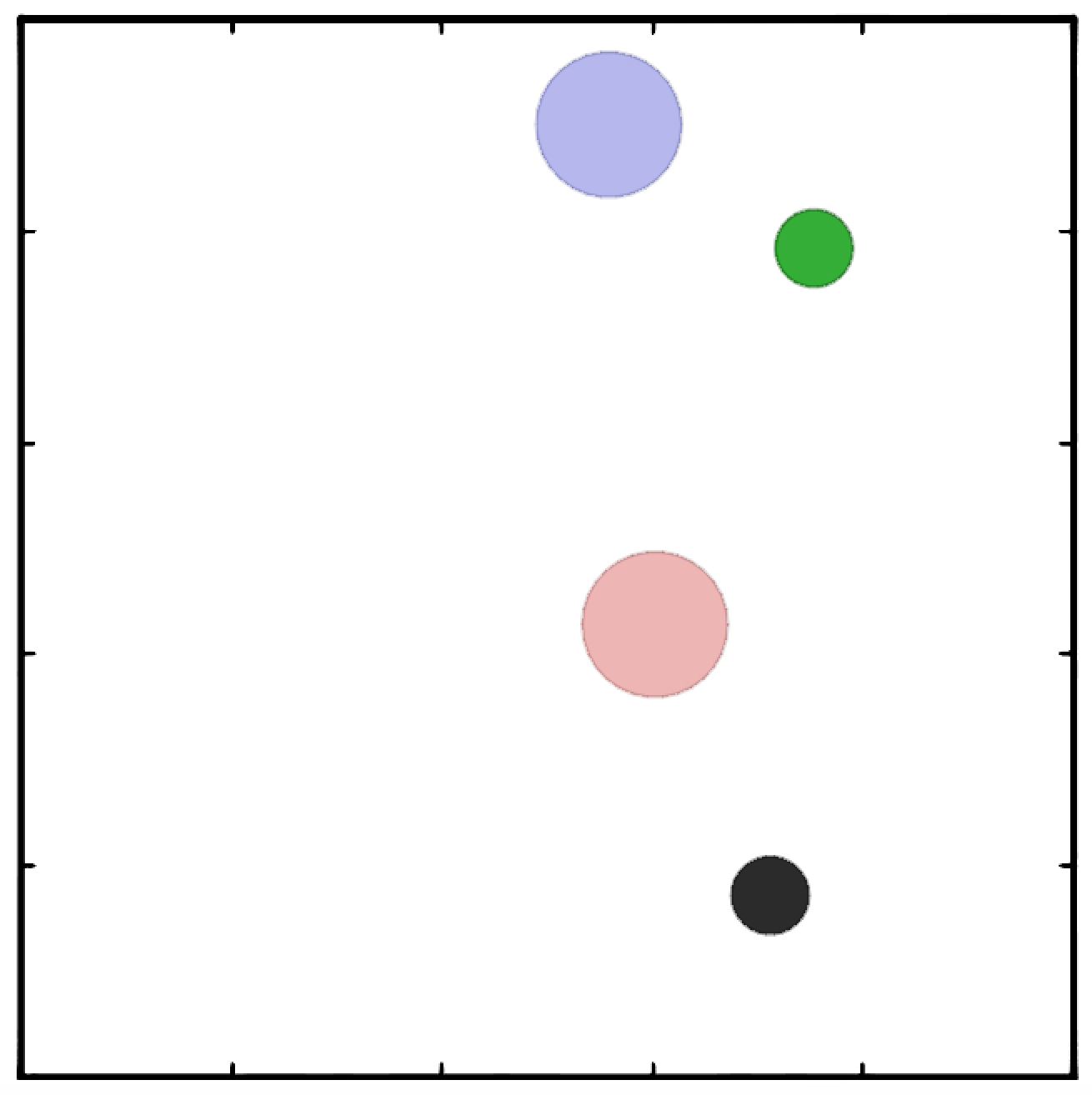}
  \includegraphics[width=0.31\linewidth]{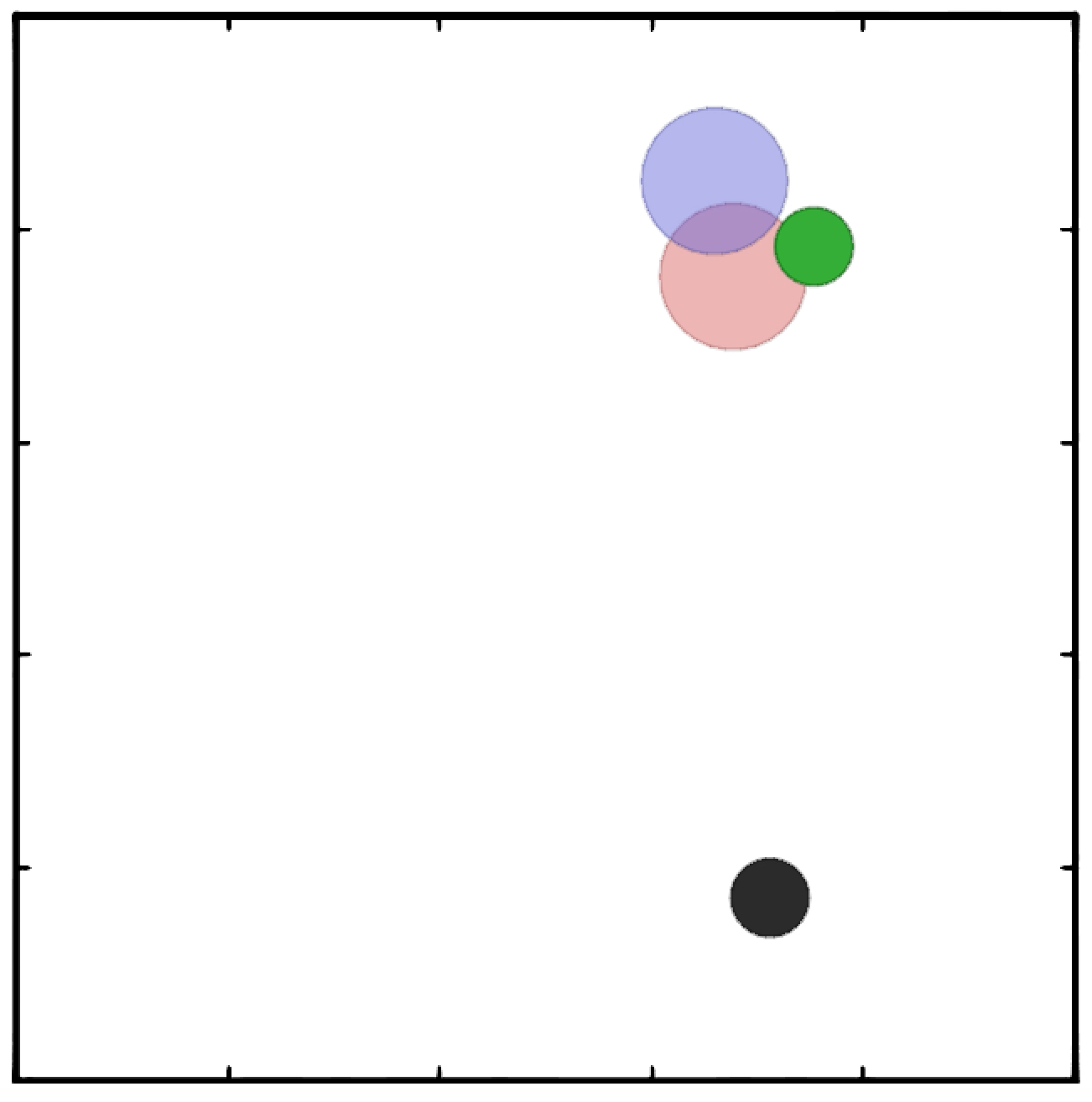}
  \caption{DDPG}
\end{subfigure}
\vspace{-2mm}
\caption{Comparison between MADDPG (left) and DDPG (right) on the cooperative communication (CC) and physical deception (PD) environments at $t=0$, $5$, and $25$. Small dark circles indicate landmarks. In CC, the grey agent is the speaker, and the color of the listener indicates the target landmark. In PD, the blue agents are trying to deceive the red adversary, while covering the target landmark (in green). MADDPG learns the correct behavior in both cases: in CC the speaker learns to output the target landmark color to direct the listener, while in PD the agents learn to cover both landmarks to confuse the adversary. DDPG (and other RL algorithms) struggles in these settings: in CC the speaker always repeats the same utterance and the listener moves to the middle of the landmarks, and in PP one agent greedily pursues the green landmark (and is followed by the adversary) while the othe agent scatters. See video for full trajectories. \vspace{-4mm}}\label{fig:disp}
\end{figure}

We hypothesize that a primary reason for the failure of traditional RL methods in this (and other) multi-agent settings is the lack of a consistent gradient signal. 
For example, if the speaker utters the correct symbol while the listener moves in the wrong direction, the speaker is penalized. This problem is exacerbated as the number of time steps grows: we observed that traditional policy gradient methods can learn when the objective of the listener is simply to reconstruct the observation of the speaker in a single time step, or if the initial positions of agents and landmarks are fixed and evenly distributed. This indicates that many of the multi-agent methods previously proposed for scenarios with short time horizons (e.g.\@ \cite{lazaridou2016multi}) may not generalize to more complex tasks.

\begin{wrapfigure}{l}{0.5\textwidth}
\centering
\includegraphics[width=\linewidth]{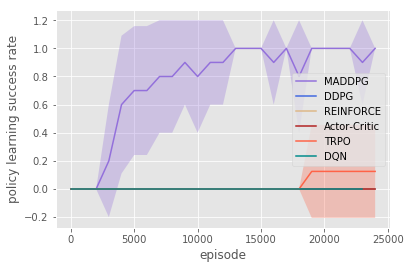}
\caption{\label{fig:comm_succ} Policy learning success rate on cooperative communication after 25000 episodes. \vspace{-4mm}}
\end{wrapfigure}
Conversely, MADDPG agents can learn coordinated behaviour more easily via the centralized critic. In the cooperative communication environment, MADDPG is able to reliably learn the correct listener and speaker policies, and the listener is often (84.0\% of the time) able to navigate to the target.

A similar situation arises for the physical deception task: when the cooperating agents are trained with MADDPG, they are able to successfully deceive the adversary by covering all of the landmarks around 94\% of the time when $L=2$ (Figure 5). Furthermore, the adversary success is quite low, especially when the adversary is trained with DDPG (16.4\% when $L=2$). This contrasts sharply with the behaviour learned by the cooperating DDPG agents, who are unable to deceive MADDPG adversaries in any scenario, and do not even deceive other DDPG agents when $L=4$.

While the cooperative navigation and predator-prey tasks have a less stark divide between success and failure, in both cases the MADDPG agents outperform the DDPG agents. In cooperative navigation, MADDPG agents have a slightly smaller average distance to each landmark, but have almost half the average number of collisions per episode (when $N=2$) compared to DDPG agents due to the ease of coordination. Similarly, MADDPG predators are far more successful at chasing DDPG prey (16.1 collisions/episode) than the converse (10.3 collisions/episode).

In the covert communication environment, we found that Bob trained with both MADDPG and DDPG out-performs Eve in terms of reconstructing Alice's message. However, Bob trained with MADDPG achieves a larger relative success rate compared with DDPG (52.4\% to 25.1\%). Further, only Alice trained with MADDPG can encode her message such that Eve achieves near-random reconstruction accuracy. 
The learning curve (a sample plot is shown in Appendix) shows that the oscillation due to the competitive nature of the environment often cannot be overcome with common decentralized RL methods. We emphasize that we do not use any of the tricks required for the cryptography environment from \cite{abadi2016learning}, including modifying Eve's loss function, alternating agent and adversary training, and using a hybrid `mix \& transform' feed-forward and convolutional architecture.


\subsection{Effect of Learning Polices of Other Agents}

We evaluate the effectiveness of learning the policies of other agents in the cooperative communication environment, following the same hyperparameters as the previous experiments and setting $\lambda=0.001$ in Eq.~\ref{eq:loss-modellearn}. The results are shown in Figure~\ref{fig:approx}. 
We observe that despite not fitting the policies of other agents perfectly (in particular, the approximate listener policy learned by the speaker has a fairly large KL divergence to the true policy), learning with approximated policies is able to achieve the same success rate as using the true policy, without a significant slowdown in convergence.

\begin{figure}
\begin{subfigure}{.5\textwidth}
  \centering
  \includegraphics[width=.8\linewidth]{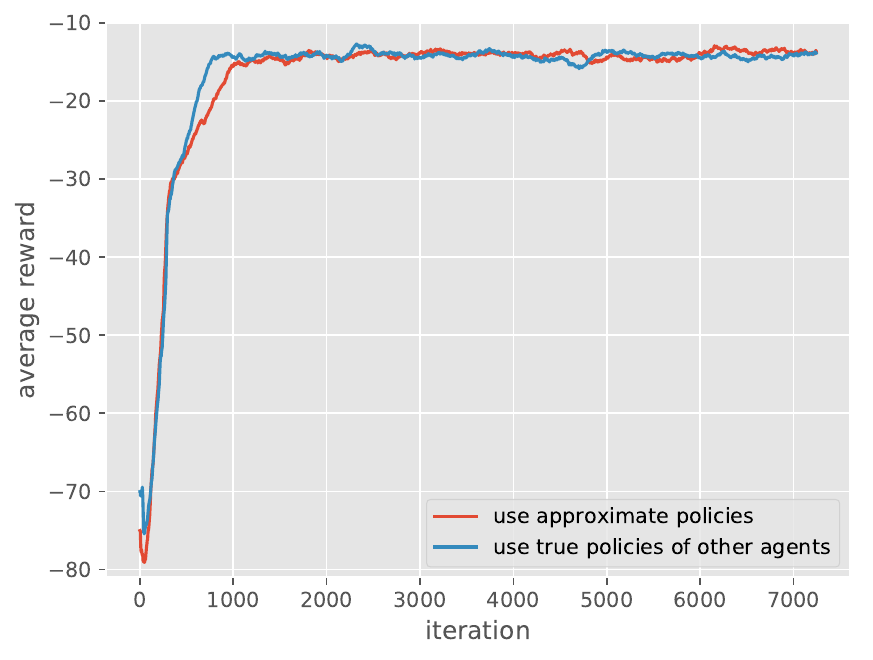}
  \label{fig:approx_reward}
\end{subfigure}
\begin{subfigure}{.5\textwidth}
  \centering
  \includegraphics[width=.8\linewidth]{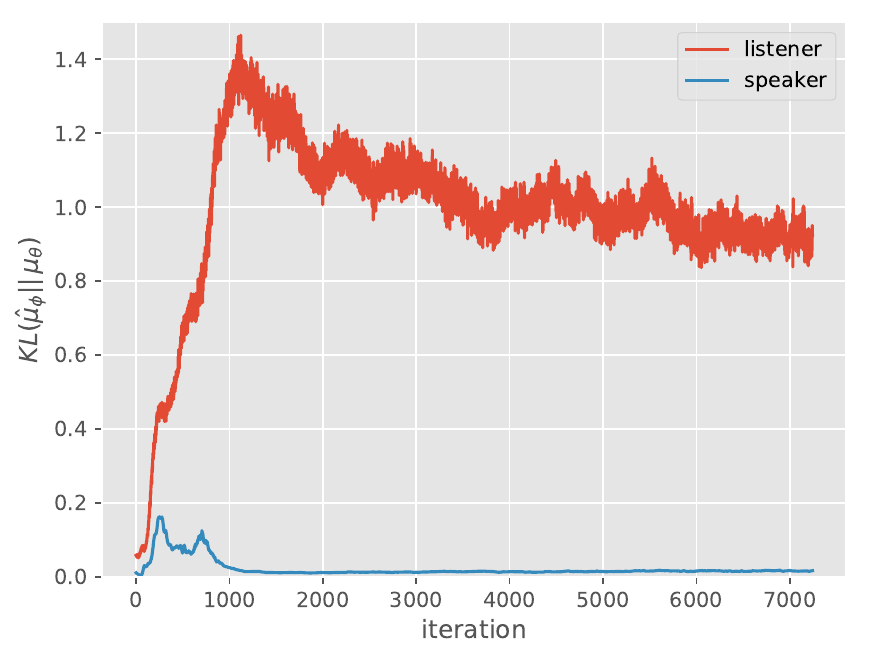}
  \label{fig:approx_kl}
\end{subfigure}
\caption{Effectiveness of learning by approximating policies of other agents in the cooperative communication scenario. \textit{Left:} plot of the reward over number of iterations; MADDPG agents quickly learn to solve the task when approximating the policies of others. \textit{Right:} KL divergence between the approximate policies and the true policies. \vspace{-3mm}}\label{fig:approx}
\end{figure}

\subsection{Effect of Training with Policy Ensembles}
We focus on the effectiveness of policy ensembles in competitive environments, including keep-away, cooperative navigation, and predator-prey. We choose $K=3$ sub-policies for the keep-away and cooperative navigation environments, and $K=2$ for predator-prey. To improve convergence speed, we enforce that the cooperative agents should have the same policies at each episode, and similarly for the adversaries.
To evaluate the approach, we measure the performance of ensemble policies and single policies in the roles of both agent and adversary. 
The results are shown on the right side of Figure~\ref{fig:adv}. We observe that agents with policy ensembles are stronger than those with a single policy. In particular, when pitting ensemble agents against single policy adversaries (second to left bar cluster), the ensemble agents outperform the adversaries by a large margin compared to when the roles are reversed (third to left bar cluster).

\section{Conclusions and Future Work}

We have proposed a multi-agent policy gradient algorithm where agents learn a centralized critic based on the observations and actions of all agents. Empirically, our method outperforms traditional RL algorithms on a variety of cooperative and competitive multi-agent environments. We can further improve the performance of our method by training agents with an ensemble of policies, an approach we believe to be generally applicable to any multi-agent algorithm.

One downside to our approach is that the input space of $Q$ grows linearly (depending on what information is contained in $\mathbf{x}$) with the number of agents $N$. This could be remedied in practice by, for example, having a modular Q function that only considers agents in a certain neighborhood of a given agent. We leave this investigation to future work. 



\subsubsection*{Acknowledgements}

The authors would like to thank Jacob Andreas, Smitha Milli, Jack Clark, Jakob Foerster, and others at OpenAI and UC Berkeley for interesting discussions related to this paper, as well as Jakub Pachocki, Yura Burda, and Joelle Pineau for comments on the paper draft. We thank Tambet Matiisen for providing the code base that was used for some early experiments associated with this paper. Ryan Lowe is supported in part by a Vanier CGS Scholarship and the Samsung Advanced Institute of Technology. Finally, we'd like to thank OpenAI for fostering an engaging and productive research environment.

\bibliographystyle{abbrv}
\bibliography{nips_2017}

\begin{thebibliography}{10}

\bibitem{googleblog}
{DeepMind AI} reduces google data centre cooling bill by 40.
\newblock
  https://deepmind.com/blog/deepmind-ai-reduces-google-data-centre-cooling-bill-40/.
\newblock Accessed: 2017-05-19.

\bibitem{abadi2016learning}
M.~Abadi and D.~G. Andersen.
\newblock Learning to protect communications with adversarial neural
  cryptography.
\newblock {\em arXiv preprint arXiv:1610.06918}, 2016.

\bibitem{boutilier96}
C.~Boutilier.
\newblock Learning conventions in multiagent stochastic domains using
  likelihood estimates.
\newblock In {\em Proceedings of the Twelfth international conference on
  Uncertainty in artificial intelligence}, pages 106--114. Morgan Kaufmann
  Publishers Inc., 1996.

\bibitem{busoniu2008comprehensive}
L.~Busoniu, R.~Babuska, and B.~De~Schutter.
\newblock A comprehensive survey of multiagent reinforcement learning.
\newblock {\em IEEE Transactions on Systems Man and Cybernetics Part C
  Applications and Reviews}, 38(2):156, 2008.

\bibitem{boutilier03}
G.~Chalkiadakis and C.~Boutilier.
\newblock Coordination in multiagent reinforcement learning: a bayesian
  approach.
\newblock In {\em Proceedings of the second international joint conference on
  Autonomous agents and multiagent systems}, pages 709--716. ACM, 2003.

\bibitem{dayan93feudal}
P.~Dayan and G.~E. Hinton.
\newblock Feudal reinforcement learning.
\newblock In {\em Advances in neural information processing systems}, pages
  271--271. Morgan Kaufmann Publishers, 1993.

\bibitem{foerster2017counterfactual}
J.~Foerster, G.~Farquhar, T.~Afouras, N.~Nardelli, and S.~Whiteson.
\newblock Counterfactual multi-agent policy gradients.
\newblock {\em arXiv preprint arXiv:1705.08926}, 2017.

\bibitem{foerster16b}
J.~N. Foerster, Y.~M. Assael, N.~de~Freitas, and S.~Whiteson.
\newblock Learning to communicate with deep multi-agent reinforcement learning.
\newblock {\em CoRR}, abs/1605.06676, 2016.

\bibitem{foerster_nonstat}
J.~N. Foerster, N.~Nardelli, G.~Farquhar, P.~H.~S. Torr, P.~Kohli, and
  S.~Whiteson.
\newblock Stabilising experience replay for deep multi-agent reinforcement
  learning.
\newblock {\em CoRR}, abs/1702.08887, 2017.

\bibitem{frank_rsa}
M.~C. Frank and N.~D. Goodman.
\newblock Predicting pragmatic reasoning in language games.
\newblock {\em Science}, 336(6084):998--998, 2012.

\bibitem{goodfellow2014generative}
I.~Goodfellow, J.~Pouget-Abadie, M.~Mirza, B.~Xu, D.~Warde-Farley, S.~Ozair,
  A.~Courville, and Y.~Bengio.
\newblock Generative adversarial nets.
\newblock In {\em Advances in neural information processing systems}, pages
  2672--2680, 2014.

\bibitem{gupta17cooperative}
J.~K. Gupta, M.~Egorov, and M.~Kochenderfer.
\newblock Cooperative multi-agent control using deep reinforcement learning.
\newblock 2017.

\bibitem{hu98}
J.~Hu and M.~P. Wellman.
\newblock Online learning about other agents in a dynamic multiagent system.
\newblock In {\em Proceedings of the Second International Conference on
  Autonomous Agents}, AGENTS '98, pages 239--246, New York, NY, USA, 1998. ACM.

\bibitem{jang2016categorical}
E.~Jang, S.~Gu, and B.~Poole.
\newblock Categorical reparameterization with gumbel-softmax.
\newblock {\em arXiv preprint arXiv:1611.01144}, 2016.

\bibitem{lauer00distributed}
M.~Lauer and M.~Riedmiller.
\newblock An algorithm for distributed reinforcement learning in cooperative
  multi-agent systems.
\newblock In {\em In Proceedings of the Seventeenth International Conference on
  Machine Learning}, pages 535--542. Morgan Kaufmann, 2000.

\bibitem{lazaridou2016multi}
A.~Lazaridou, A.~Peysakhovich, and M.~Baroni.
\newblock Multi-agent cooperation and the emergence of (natural) language.
\newblock {\em arXiv preprint arXiv:1612.07182}, 2016.

\bibitem{multiagent_ssd}
J.~Z. Leibo, V.~F. Zambaldi, M.~Lanctot, J.~Marecki, and T.~Graepel.
\newblock Multi-agent reinforcement learning in sequential social dilemmas.
\newblock {\em CoRR}, abs/1702.03037, 2017.

\bibitem{levine2015end}
S.~Levine, C.~Finn, T.~Darrell, and P.~Abbeel.
\newblock End-to-end training of deep visuomotor policies.
\newblock {\em arXiv preprint arXiv:1504.00702}, 2015.

\bibitem{lillicrap2015continuous}
T.~P. Lillicrap, J.~J. Hunt, A.~Pritzel, N.~Heess, T.~Erez, Y.~Tassa,
  D.~Silver, and D.~Wierstra.
\newblock Continuous control with deep reinforcement learning.
\newblock {\em arXiv preprint arXiv:1509.02971}, 2015.

\bibitem{littman1994markov}
M.~L. Littman.
\newblock Markov games as a framework for multi-agent reinforcement learning.
\newblock In {\em Proceedings of the eleventh international conference on
  machine learning}, volume 157, pages 157--163, 1994.

\bibitem{matignon12coordinated}
L.~Matignon, L.~Jeanpierre, A.-I. Mouaddib, et~al.
\newblock Coordinated multi-robot exploration under communication constraints
  using decentralized markov decision processes.
\newblock In {\em AAAI}, 2012.

\bibitem{hyst07}
L.~Matignon, G.~J. Laurent, and N.~Le~Fort-Piat.
\newblock Hysteretic q-learning: an algorithm for decentralized reinforcement
  learning in cooperative multi-agent teams.
\newblock In {\em Intelligent Robots and Systems, 2007. IROS 2007. IEEE/RSJ
  International Conference on}, pages 64--69. IEEE, 2007.

\bibitem{matignon12independent}
L.~Matignon, G.~J. Laurent, and N.~Le~Fort-Piat.
\newblock Independent reinforcement learners in cooperative markov games: a
  survey regarding coordination problems.
\newblock {\em The Knowledge Engineering Review}, 27(01):1--31, 2012.

\bibitem{mnih2015human}
V.~Mnih, K.~Kavukcuoglu, D.~Silver, A.~A. Rusu, J.~Veness, M.~G. Bellemare,
  A.~Graves, M.~Riedmiller, A.~K. Fidjeland, G.~Ostrovski, et~al.
\newblock Human-level control through deep reinforcement learning.
\newblock {\em Nature}, 518(7540):529--533, 2015.

\bibitem{mordatch2017emergence}
I.~Mordatch and P.~Abbeel.
\newblock Emergence of grounded compositional language in multi-agent
  populations.
\newblock {\em arXiv preprint arXiv:1703.04908}, 2017.

\bibitem{hyst17}
S.~Omidshafiei, J.~Pazis, C.~Amato, J.~P. How, and J.~Vian.
\newblock Deep decentralized multi-task multi-agent reinforcement learning
  under partial observability.
\newblock {\em CoRR}, abs/1703.06182, 2017.

\bibitem{panait05}
L.~Panait and S.~Luke.
\newblock Cooperative multi-agent learning: The state of the art.
\newblock {\em Autonomous Agents and Multi-Agent Systems}, 11(3):387--434, Nov.
  2005.

\bibitem{peng17starcraft}
P.~Peng, Q.~Yuan, Y.~Wen, Y.~Yang, Z.~Tang, H.~Long, and J.~Wang.
\newblock Multiagent bidirectionally-coordinated nets for learning to play
  starcraft combat games.
\newblock {\em CoRR}, abs/1703.10069, 2017.

\bibitem{alphago}
D.~Silver, A.~Huang, C.~J. Maddison, A.~Guez, L.~Sifre, G.~van~den Driessche,
  J.~Schrittwieser, I.~Antonoglou, V.~Panneershelvam, M.~Lanctot, S.~Dieleman,
  D.~Grewe, J.~Nham, N.~Kalchbrenner, I.~Sutskever, T.~Lillicrap, M.~Leach,
  K.~Kavukcuoglu, T.~Graepel, and D.~Hassabis.
\newblock {Mastering the game of Go with deep neural networks and tree search}.
\newblock {\em Nature}, 529(7587):484 -- 489, 2016.

\bibitem{silver2014deterministic}
D.~Silver, G.~Lever, N.~Heess, T.~Degris, D.~Wierstra, and M.~Riedmiller.
\newblock Deterministic policy gradient algorithms.
\newblock In {\em Proceedings of the 31st International Conference on Machine
  Learning}, pages 387--395, 2014.

\bibitem{sukhbaatar2016learning}
S.~Sukhbaatar, R.~Fergus, et~al.
\newblock Learning multiagent communication with backpropagation.
\newblock In {\em Advances in Neural Information Processing Systems}, pages
  2244--2252, 2016.

\bibitem{sukhbaatar2017intrinsic}
S.~Sukhbaatar, I.~Kostrikov, A.~Szlam, and R.~Fergus.
\newblock Intrinsic motivation and automatic curricula via asymmetric
  self-play.
\newblock {\em arXiv preprint arXiv:1703.05407}, 2017.

\bibitem{sutton1998reinforcement}
R.~S. Sutton and A.~G. Barto.
\newblock {\em Reinforcement learning: An introduction}, volume~1.
\newblock MIT press Cambridge, 1998.

\bibitem{sutton2000policy}
R.~S. Sutton, D.~A. McAllester, S.~P. Singh, and Y.~Mansour.
\newblock Policy gradient methods for reinforcement learning with function
  approximation.
\newblock In {\em Advances in neural information processing systems}, pages
  1057--1063, 2000.

\bibitem{tampuu2017multiagent}
A.~Tampuu, T.~Matiisen, D.~Kodelja, I.~Kuzovkin, K.~Korjus, J.~Aru, J.~Aru, and
  R.~Vicente.
\newblock Multiagent cooperation and competition with deep reinforcement
  learning.
\newblock {\em PloS one}, 12(4):e0172395, 2017.

\bibitem{tan93multi}
M.~Tan.
\newblock Multi-agent reinforcement learning: Independent vs. cooperative
  agents.
\newblock In {\em Proceedings of the tenth international conference on machine
  learning}, pages 330--337, 1993.

\bibitem{hyper_q}
G.~Tesauro.
\newblock Extending q-learning to general adaptive multi-agent systems.
\newblock In {\em Advances in neural information processing systems}, pages
  871--878, 2004.

\bibitem{thomas2011conjugate}
P.~S. Thomas and A.~G. Barto.
\newblock Conjugate markov decision processes.
\newblock In {\em Proceedings of the 28th International Conference on Machine
  Learning (ICML-11)}, pages 137--144, 2011.

\bibitem{williams1992simple}
R.~J. Williams.
\newblock Simple statistical gradient-following algorithms for connectionist
  reinforcement learning.
\newblock {\em Machine learning}, 8(3-4):229--256, 1992.

\end{thebibliography}


\newpage

\section*{Appendix}
\label{sec:appendix}

\subsection*{Multi-Agent Deep Deterministic Policy Gradient Algorithm}
For completeness, we provide the MADDPG algorithm below.

\begin{algorithm}[H]
 \SetAlgoLined
  \begin{algorithmic}
    \FOR{$\textrm{episode}=1\textrm{ to }M$}
      \STATE Initialize a random process $\mathcal{N}$ for action exploration
      \STATE Receive initial state $\mathbf{x}$
      \FOR{$t=1\textrm{ to max-episode-length}$}
          \STATE for each agent $i$, select action $a_i=\cpol_{\theta_i}(o_i)+\mathcal{N}_t$ w.r.t. the current policy and exploration
          \STATE Execute actions $a=(a_1,\ldots,a_N)$ and observe reward $r$ and new state $\mathbf{x}'$
          \STATE Store $(\mathbf{x},a,r,\mathbf{x}')$ in replay buffer $\mathcal{D}$
          \STATE $\mathbf{x}\gets\mathbf{x}'$
          \FOR{agent $i=1\textrm{ to }N$}
            \STATE Sample a random minibatch of $S$ samples $(\mathbf{x}^j,a^j,r^j,\mathbf{x}'^j)$ from $\mathcal{D}$
            \STATE Set $y^j=r_i^j+\gamma\, Q_i^{\cpol'}(\mathbf{x}'^j,a_1',\ldots,a_N')|_{a_k'=\cpol'_k(o_k^j)}$
            \STATE Update critic by minimizing the loss $\mathcal{L}(\theta_i)=\frac{1}{S}\sum_j\left(y^j-Q_i^{\cpol}(\mathbf{x}^j,a^j_1,\ldots,a^j_N)\right)^2$
            \STATE Update actor using the sampled policy gradient:
            $$
            \nabla_{\theta_i}J\approx\frac{1}{S}\sum_j\nabla_{\theta_i}\cpol_i(o_i^j)\nabla_{a_i}Q_i^{\cpol}(\mathbf{x}^j,a^j_1,\ldots,a_i,\ldots,a^j_N)\big|_{a_i=\cpol_i(o_i^j)}
            $$
          \ENDFOR     
          \STATE Update target network parameters for each agent $i$:
          $$
          \theta_i'\gets\tau\theta_i+(1-\tau)\theta_i'
          $$
      \ENDFOR
    \ENDFOR
  \end{algorithmic}
 \caption{Multi-Agent Deep Deterministic Policy Gradient for $N$ agents}
\end{algorithm}

\subsection*{Experimental Results}
\label{sec:tables}

In all of our experiments, we use the Adam optimizer with a learning rate of 0.01 and $\tau=0.01$ for updating the target networks. $\gamma$ is set to be 0.95. The size of the replay buffer is $10^6$ and we update the network parameters after every 100 samples added to the replay buffer. We use a batch size of 1024 episodes before making an update, except for TRPO where we found a batch size of 50 lead to better performance (allowing it more updates relative to MADDPG). We train with 10 random seeds for environments with stark success/ fail conditions (cooperative communication, physical deception, and covert communication) and 3 random seeds for the other environments.

The details of the experimental results are shown in the following tables.

\begin{table*}[ht!]
\small
\centering
\begin{tabular}{l c c}
\toprule
Agent $\pol$& \textbf{Target reach \% }&    \textbf{Average distance}   \\ \hline
MADDPG &  \textbf{84.0\%} & \textbf{0.133}  \\
DDPG &  32.0\% & 0.456 \\
DQN & 24.8\% & 0.754 \\
Actor-Critic & 17.2\% & 2.071 \\
TRPO & 20.6\% & 1.573 \\
REINFORCE &  13.6\%  & 3.333 \\
\bottomrule
\end{tabular}
\caption{\label{tab:simple_speaker_listener} Percentage of episodes where the agent reached the target landmark and average distance from the target in the cooperative communication environment, after 25000 episodes. Note that the percentage of targets reached is different than the policy learning success rate in Figure \ref{fig:comm_succ}, which indicates the percentage of runs in which the correct policy was learned (consistently reaching the target landmark). Even when the correct behavior is learned, agents occasionally hover slightly outside the target landmark on some episodes, and conversely agents who learn to go to the middle of the landmarks occasionally stumble upon the correct landmark.}
\end{table*}

\begin{table*}[ht!]
\small
\centering
\begin{tabular}{l c c c c}
\toprule
& \multicolumn{2}{c}{ $N = 3$} & \multicolumn{2}{c}{ $N = 6$} \\
Agent $\pol$& \textbf{Average dist.}&    \textbf{\# collisions} & \textbf{Average dist.}&    \textbf{\# collisions}   \\ \hline
MADDPG &  \textbf{1.767} & \textbf{0.209} & \textbf{3.345} & \textbf{1.366}  \\
DDPG &  1.858 & 0.375 & 3.350 & 1.585 \\
\bottomrule
\end{tabular}
\caption{\label{tab:simple_spread} Average \# of collisions per episode and average agent distance from a landmark in the cooperative navigation task, using 2-layer 128 unit MLP policies.}
\end{table*}

\begin{table*}[ht!]
\small
\centering
\begin{tabular}{l l c c}
\toprule
Agent $\pol$ & Adversary $\pol$ & \textbf{\# touches ({\sc pp1}) }&    \textbf{\# touches ({\sc pp2})} \\ \hline
MADDPG & MADDPG & 11.0 & 0.202  \\
MADDPG & DDPG &  \textbf{16.1} & \textbf{0.405} \\
DDPG & MADDPG & 10.3  & 0.298  \\
DDPG & DDPG & 9.4 & 0.321  \\
\bottomrule
\end{tabular}
\caption{\label{tab:simple_tag} Average number of prey touches by predator per episode on two predator-prey environments with $N=L=3$, one where the prey (adversaries) are slightly (30\%) faster ({\sc pp1}), and one where they are significantly (100\%) faster ({\sc pp2}). All policies in this experiment are 2-layer 128 unit MLPs. }
\end{table*}

\begin{table*}[ht!]
\fontsize{8.5}{9}\selectfont
\centering
\begin{tabular}{l l c c c c c c}
\toprule
 & & \multicolumn{3}{c}{ $N = 2$} & \multicolumn{3}{c}{ $N=4$} \\
Agent $\pol$ & Adversary $\pol$ & \textbf{AG succ \% }&    \textbf{ADV succ \%}& \textbf{$\Delta$ succ \%}& \textbf{AG succ \% }&    \textbf{ADV succ \%}& \textbf{$\Delta$ succ \%} \\ \hline
MADDPG & MADDPG & 94.4\% & 39.2\% & 55.2\% & 81.5\% & 28.3\% & \textbf{53.2\%} \\
MADDPG & DDPG &  92.2\% & 16.4\% & \textbf{75.8\%} & 69.6\% & 19.8\% & 49.4\% \\
DDPG & MADDPG & 68.9\%  & 59.0\% & 9.9\% & 35.7\% & 32.1\% & 3.6\% \\
DDPG & DDPG & 74.7\%  & 38.6\% & 36.1\% & 18.4\% & 35.8\% & -17.4\% \\
\bottomrule
\end{tabular}
\caption{\label{tab:simple_adversary} Results on the physical deception task, with $N=2$ and $4$ cooperative agents/landmarks. Success (\textit{succ \%}) for agents (AG) and adversaries (ADV) is if they are within a small distance from the target landmark. }
\end{table*}

\begin{table*}[ht!]
\small
\centering
\begin{tabular}{l l c c c}
\toprule
Alice, Bob $\pol$& Eve $\pol$& \textbf{Bob succ \% }&    \textbf{Eve succ \%} & \textbf{$\Delta$ succ \%}  \\ \hline
MADDPG & MADDPG & 96.5\% & 52.1\% & 44.4\% \\
MADDPG & DDPG & 96.8\% & 44.4\% & \textbf{52.4\%} \\
DDPG & MADDPG & 65.3\% & 64.3\% & 1.0\% \\
DDPG & DDPG &  92.7\% & 67.6\% & 25.1\%\\
\bottomrule
\end{tabular}
\caption{\label{tab:simple_crypto} Agent (Bob) and adversary (Eve) success rate (\textit{succ\@ \%}, i.e.\@ correctly reconstructing the speaker's message) in the covert communication environment. The input message is drawn from a set of two 4-dimensional one-hot vectors.}
\end{table*}

\begin{table}[ht!]
\centering
{\small
\begin{subtable}{.3\linewidth}
      \centering
        \begin{tabular}{c c c}
        \toprule
            & \emph{S}. AG. &  \emph{E}. AG.\\
            \hline
            \emph{S}. Adv.&7.94&7.74\\
            \emph{E}. Adv.&8.35&8.11\\
            \bottomrule
        \end{tabular}
        \caption{KA: average frames that the adversary occupies the goal. For Adv., the larger the better.}
\end{subtable}
\hspace{2mm}
\begin{subtable}{.3\linewidth}
      \centering
        \begin{tabular}{c c c}
        \toprule
            & \emph{S}. AG. & \emph{E}. AG.\\
            \hline
            \emph{S}. Adv.&4.25&4.10\\
            \emph{E}. Adv.&5.55&4.44\\
        \bottomrule
        \end{tabular}
        \caption{PD: average frames that the adversary stays at the goal. For Adv., the larger the better.}
\end{subtable}
\hspace{2mm}
\begin{subtable}{.3\linewidth}
      \centering
        \begin{tabular}{c c c}
        \toprule
            & \emph{S}. AG. & \emph{E}. AG.\\
            \hline
            \emph{S}. Adv.&0.201&0.211\\
            \emph{E}. Adv. &0.125&0.17\\
            \bottomrule
        \end{tabular}
        \caption{PP: average number of collisions. For Adv., the smaller the better.}
\end{subtable}
}
\caption{Evaluations of the adversary agent w./w.o. policy ensembles over 1000 trials on different scenarios including (a) keep-away (KA) with $N=M=1$, (b) physical deception (PD) with $N=2$ and (c) predator-prey (PP) with $N=4$ and $L=1$. \emph{S.} denotes agents with a single policy. \emph{E.} denotes agents with policy ensembles.}\label{tab:ensemble}
\end{table}

\begin{figure}
\centering
\includegraphics[width=.65\linewidth]{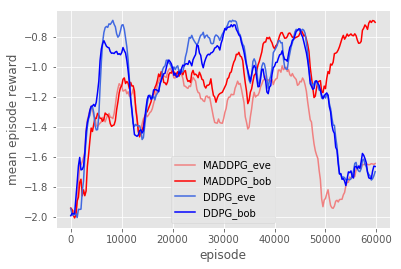}
\caption{\label{fig:crypto} In competitive environments such as `covert communication', the reward can oscillate significantly as agents adapt to each other. DDPG is often unable to overcome this, whereas our MADDPG algorithm has much greater success. \vspace{-3mm}}
\end{figure}

\newpage

\subsection*{Variance of Policy Gradient Algorithms in a Simple Multi-Agent Setting}
\label{sec:variance}
To analyze the variance of policy gradient methods in multi-agent settings, we consider a simple cooperative scenario with $N$ agents and binary actions: $P(a_i=1) = \theta_i$. We define the reward to be $1$ if all actions are the same $a_1=a_2=\ldots=a_N$, and $0$ otherwise. This is a simple scenario with no temporal component: agents must simply learn to either always output $1$ or always output $0$ at each time step. Despite this, we can show that the probability of taking a gradient step in the correct direction decreases exponentially with the number of agents $N$.

\paragraph{Proposition 1.}
\textit{
Consider $N$ agents with binary actions: $P(a_i=1) = \theta_i$, where $R(a_1,\dots,a_N) = \mathbf{1}_{a_1=\dots=a_N}$. We assume an uninformed scenario, in which agents are initialized to $\theta_i=0.5 \ \forall i$. Then, if we are estimating the gradient of the cost $J$ with policy gradient, we have:
$$P(\langle \hat{\nabla} J, \nabla J \rangle > 0) \propto (0.5)^N,$$
where $\hat{\nabla} J$ is the policy gradient estimator from a single sample, and $\nabla J$ is the true gradient.
}

\begin{proof}
We can write $P(a_i) = {\theta_i}^{a_i}  (1-\theta_i)^{1-a_i}$, and $\log P(a_i) = a_i \log {\theta_i} + (1-a_i)\log(1-\theta_i)$.

The policy gradient estimator (from a single sample) is:
\begin{equation}
\begin{split}
\hat{\dd} J &= R(a_1,\dots,a_N) \dd \log P(a_1,\dots,a_N) \\ 
& = R(a_1,\dots,a_N) \dd \sum_i a_i \log {\theta_i} + (1-a_i)\log(1-\theta_i) \\
& = R(a_1,\dots,a_N) \dd (a_i \log {\theta_i} + (1-a_i)\log(1-\theta_i)) \\
& = R(a_1,\dots,a_N) \left(\frac{a_i}{\theta_i} - \frac{1-a_i}{1-\theta_i} \right)
\end{split}    
\end{equation}
For $\theta_i = 0.5$ we have: 
$$
\hat{\dd} J = R(a_1,\dots,a_N) \left(2 a_i- 1\right)
$$
And the expected reward can be calculated as:
$$
\mathbb E (R) = \sum_{a_1,\dots,a_N} R(a_1,\dots,a_N) (0.5)^N
$$

Consider the case where $R(a_1,\dots,a_N) = \mathbf{1}_{a_1=\dots=a_N=1}$. Then 
$$
\mathbb E (R) = (0.5)^N
$$
and 
$$
\mathbb E (\hat{\dd} J ) = \dd J = (0.5)^N
$$
The variance of a single sample of the gradient is then:
$$
\mathbb V (\hat{\dd} J ) = \mathbb E (\hat{\dd} J^2 ) - \mathbb E (\hat{\dd} J ) ^2 = (0.5)^N - (0.5)^{2N}
$$
What is the probability of taking a step in the right direction? We can look at $P(\langle \hat{\nabla} J, \nabla J \rangle > 0)$. We have:
$$
\langle \hat{\nabla} J, \nabla J \rangle = \sum_i \hat{\dd} J  \times (0.5)^N = (0.5)^N \sum_i \hat{\dd} J,
$$
so $P(\langle \hat{\nabla} J, \nabla J \rangle > 0) = (0.5)^N$. Thus, as the number of agents increases, the probability of taking a gradient step in the right direction decreases exponentially.
\end{proof}

While this is a somewhat artificial example, it serves to illustrate that there are simple environments that become progressively more difficult  (in terms of the probability of taking a gradient step in a direction that increases reward) for policy gradient methods as the number of agents grows. This is particularly true in environments with sparse rewards, such as the one described above. Note that in this example, the policy gradient variance $\mathbb V (\hat{\dd} J )$ actually decreases as N grows. However, the expectation of the policy gradient decreases as well, and the signal to noise ratio $\mathbb E (\hat{\dd} J ) / (\mathbb V (\hat{\dd} J ))^{1/2}$ decreases with $N$, corresponding to the decreasing probability of a correct gradient direction. 
The intuitive reason a centralized critic helps reduce the variance of the gradients is that we remove a source of uncertainty; conditioned only on the agent's own actions, there is significant variability associated with the actions of other agents, which is largely removed when using these actions as input to the critic.







\end{document}